\documentclass[11pt]{article}

\usepackage[a4paper,margin=1in]{geometry}
\usepackage{setspace}
\usepackage{amsmath,amssymb,amsfonts,amsthm}
\usepackage{bm}
\usepackage{mathtools}

\usepackage{graphicx}
\usepackage{subcaption}
\usepackage{booktabs}

\usepackage{algorithm}
\usepackage{algorithmic}

\usepackage[numbers]{natbib}

\usepackage[colorlinks=true,
            linkcolor=blue,
            citecolor=blue,
            urlcolor=blue]{hyperref}

\newtheorem{theorem}{Theorem}
\newtheorem{lemma}{Lemma}

\newtheorem{corollary}{Corollary}
\theoremstyle{definition}
\newtheorem{definition}{Definition}
\theoremstyle{remark}
\newtheorem{remark}{Remark}
\newtheorem{assumption}[theorem]{Assumption}
\newtheorem{example}[theorem]{Example}

\DeclareMathOperator*{\argmax}{\arg\,\max}
\DeclareMathOperator*{\argmin}{\arg\,\min}

\makeatletter
\newcommand{\rmnum}[1]{\romannumeral #1}
\newcommand{\Rmnum}[1]{\expandafter\@slowromancap\romannumeral #1@}
\makeatother

\title{Inverse Contextual Bandits without Rewards: \\ Learning from a Non-Stationary Learner via Suffix Imitation}
\author{
  Yuqi Kong \qquad
  Xiao Zhang \qquad
  Weiran Shen \\
  Gaoling School of Artificial Intelligence, Renmin University of China\\
  \texttt{\{kongyuqi,zhangx89,shenweiran\}@ruc.edu.cn}
}

\date{} 

\begin{document}
\maketitle

\begin{abstract}
We study the Inverse Contextual Bandit (ICB) problem, in which a learner seeks to optimize a policy while an observer, who cannot access the learner's rewards and only observes actions, aims to recover the underlying problem parameters. During the learning process, the learner's behavior naturally transitions from exploration to exploitation, resulting in non-stationary action data that poses significant challenges for the observer. To address this issue, we propose a simple and effective framework called Two-Phase Suffix Imitation. The framework discards data from an initial burn-in phase and performs empirical risk minimization using only data from a subsequent imitation phase. We derive a predictive decision loss bound that explicitly characterizes the bias-variance trade-off induced by the choice of burn-in length. Despite the severe information deficit, we show that a reward-free observer can achieve a convergence rate of $\tilde O(1/\sqrt{N})$, matching the asymptotic efficiency of a fully reward-aware learner. This result demonstrates that a passive observer can effectively uncover the optimal policy from actions alone, attaining performance comparable to that of the learner itself.
\end{abstract}


\section{Introduction}
Contextual bandits have emerged as a fundamental framework for sequential decision-making under uncertainty, proving instrumental in domains such as recommendation systems and clinical trials \cite{contextualforhealth, contextforrecommandation}. 
Classical work focuses on designing algorithms that maximize cumulative rewards. Yet in modern deployments, performance alone is often insufficient: agents may need to verify why a system prefers particular actions to others, and whether it is still exploring. 
This is particularly challenging when rewards are delayed, private, or unobservable, leaving only the interaction logs as the record. We term this problem the inverse contextual bandit: given a history of interactions generated by a bandit algorithm, can an observer recover the underlying optimal policy without access to the reward signal?

This perspective is related to Inverse Reinforcement Learning (IRL) and Behavior Cloning (BC), which also aim to explain behavior from demonstrations. However, a fundamental mismatch renders standard IRL methods ineffective. Traditional IRL typically assumes that decisions are generated by a stationary and optimal expert. In contrast, bandit logs are produced by an adaptive and evolving learning agent, which is generating data through a learning process from ignorance to expertise. Applying standard BC algorithms here is problematic: they indiscriminately regard these initial exploratory decisions as optimal demonstrations, leading to the imitation of noisy, suboptimal data.

The paradigm of ``learning a learner'' poses two core technical challenges. First, the observer operates in a reward-free manner: it observes contexts and the learning agent’s chosen actions, but not the realized rewards. Second, the induced dataset is non-stationary: the label quality varies over time as the learner learns, so early rounds can be much noisier than later ones. Naively pooling all observations can allow low-quality early data to dominate and degrade the recovered policy.

To address these challenges, we propose a simple yet effective Two-Phase Suffix Imitation framework. The key insight is counterintuitive but crucial: less data can be better data. Because ignoring the exploratory prefix can substantially improve the signal to noise ratio for imitation. 
Our theoretical analysis reveals a surprising and encouraging result. We prove that even under a strict information deficit, our reward-free observer achieves a regret bound performance comparable to the learner (Corollary~\ref{cor:conservative_burnin}).

Our main contributions are summarized as follows: 
\begin{itemize}
    \item We formalize the Inverse Contextual Bandit setting where an observer learns from a non-stationary learner with no reward feedback. We propose the Two-Phase Suffix Imitation framework, which introduces a discarding strategy to address the distribution shift caused by exploration.
    \item We provide a finite-sample analysis of the proposed method showing that the observer recovers the optimal policy with vanishing predictive regret. Crucially, we show that the observer's efficiency is robust to the learner's regret rate.
    \item We conduct extensive experiments to validate our framework. We find that despite the significant lack of information, the observer's parameter estimation error converges to a level comparable to that of a fully reward-aware learner.
\end{itemize}
\section{Related Work}
Our work sits at the intersection of contextual bandits, inverse learning from demonstrations, and learning under weak supervision. We highlight how our inverse contextual bandit (ICB) setting departs from standard formulations: the observer sees only context–action logs (no rewards), and the demonstrations are generated by a learning agent whose behavior is non-stationary due to exploration.

{\bfseries Contextual bandits.}
Contextual bandits model repeated decisions with side information and bandit feedback. Classic algorithms such as LinUCB \cite{li2010contextual}, Thompson sampling for linear models (LinTS) \cite{agrawal2013thompson}, and OFUL \cite{abbasi2011improved} achieve strong empirical performance and no-regret guarantees under standard assumptions \cite{li2010contextual, Lattimore2020BanditA}. This literature assumes the decision maker observes rewards and focuses on designing policies that minimize regret. In contrast, we study the inverse problem: given only the context–action trace produced by such a learner, an external observer aims to recover a deployable decision rule that performs well with respect to the same optimality benchmark—without any reward feedback.

{\bfseries Inverse reinforcement learning and imitation learning.}
ICB is related to inverse reinforcement learning (IRL) and imitation learning (IL), which seek to explain or replicate behavior from demonstrations \cite{ng2000algorithms, abbeel2004apprenticeship, zare2024survey}. A key difference is the information structure and the data generating process. Many IRL/IL methods assume demonstrations are generated by a stationary (near-)optimal expert and often leverage rewards, environment dynamics, or additional supervision to handle suboptimality \cite{zeng2023demonstrations, seo2024mitigating, maoffline}. While recent work relaxes optimality by learning from noisy or imperfect demonstrations \cite{belkhale2023data, hoang2024sprinql, yue2024leverage}, many approaches still require reward annotations \cite{xu2024provably} or external preference/ranking signals over trajectories \cite{kang2023beyond, choi2024listwise}. Our setting is strictly reward-free: the observer only sees interaction logs. Rather than relying on rewards or rankings, we exploit a structural property unique to “learning a learner”: label quality improves over time as exploration gives way to more stable behavior, motivating suffix-based imitation.

{\bfseries Inverse contextual bandits and behavioral evolution.}
A growing body of work on inverse contextual bandits aims to recover how behavior evolves over time, emphasizing interpretability of non-stationary policies and their dynamics \cite{huyuk2022inverse, IBCB}. We take a more pragmatic stance. Instead of reconstructing the full evolution of the learner, we focus on recovering a single deployable policy from the observation trace. Concretely, we discard an initial burn-in prefix and perform suffix imitation, and we provide generalization based performance guarantees for the resulting policy. The goal is not to exhaust performance limits, but to show that a simple, implementable procedure already yields strong and predictable gains in reward-free observation settings.
\section{Preliminary}

\subsection{Linear Contextual Bandits} 
We consider the stochastic linear contextual bandit problem with an unknown parameter $\theta^*\in \mathbb R^d$. We assume, without loss of generality, that $\|\theta^*\|_2\leq 1$. Let $\mathcal A=[K]$ be the finite set of possible arms. At each round $t\in[N]$, where N is the time horizon, the environment reveals a feasible arm set $\mathcal{A}_t\subseteq \mathcal A$ along with their associated feature vectors $X_t = \{x_a \in\mathbb R^d: a\in \mathcal{A}_t\}$. We assume that $\|x_a\|_2\leq 1$ for all $t,a$. We assume that the sequence of contexts $\{(\mathcal{A}_t, X_t)\}_{t\geq 1}$ is i.i.d. sampled from a distribution $D$ over the context space. Upon observing the context, the learner selects an arm $a\in \mathcal{A}_t$ at round $t$ and then observes a reward: $r_t(a)=\langle x_a, \theta^* \rangle+\xi_t$, where $\xi_t$ is a noise term. Let $\mathcal{F}_{t-1}$ denote the history of observations prior to observing $r_t$. We assume that $\{\xi_t\}_{t \ge 1}$ is conditionally $R$-sub-Gaussian given $\mathcal{F}_{t-1}$, i.e., $\mathbb{E}[\xi_t | \mathcal{F}_{t-1}] = 0$ and $\mathbb{E}[\exp(\lambda \xi_t) | \mathcal{F}_{t-1}] \le \exp(\lambda^2 R^2 / 2)$. This is a standard assumption in the stochastic linear bandit literature. An optimal arm at round $t$ is any: $a_t^* \in \argmax_{a\in \mathcal{A}_t} \langle x_a, \theta^* \rangle$,
with ties broken by a fixed rule.
\begin{definition}
    The instantaneous regret at round $t$ is defined as the difference between the expected reward of the optimal arm and the selected arm $a_t$:
\[
    \Delta_t(a_t):=\langle x_{a^*_t}, \theta^* \rangle -\langle x_{a_t}, \theta^* \rangle \geq 0.
\]
\end{definition}

Consider the finite arm set $\mathcal A$ and the fixed parameter $\theta^*$, we introduce the following uniform margin condition to facilitate the analysis of mistake bounds and ensure the uniqueness of the optimal action.

\begin{definition}\label{def:regretgap} 
We define the minimum gap $\Delta_{\min}$ as the minimal reward difference between the optimal arm and the best suboptimal arm across all possible contexts:
\begin{equation*}
    \Delta_{\min} := \inf_{(\mathcal{A}_t, X_t) \in \mathrm{supp}(\mathcal{D})} \left[\min_{a\in \mathcal A_t\setminus\{a_t^*\}}\Delta_t(a)\right].
\end{equation*}
\end{definition}

A policy $\pi=\{\pi_t\}_{t=1}^N$ consists of decision rules that, at each round $t$, given the current feasible set and features $(\mathcal A_t, X_t)$, the rule $\pi_t$ outputs a distribution over feasible arms, 
and the action is then drawn as $a_t\sim \pi_t(\cdot| \mathcal A_t, X_t)$.
For any policy $\pi$ that maps $(\mathcal A_t, X_t)$ to an action $a_t$ in $\mathcal A_t$, define its $N$-round regret for horizon $N$ as
\begin{align*}
    R_N(\pi):=\sum_{t=1}^N\left(\langle x_{a_t^*},\theta^* \rangle - \langle x_{a_t},\theta^* \rangle\right).
\end{align*}

\subsection{Inverse Contextual Bandits}
We have formalized the environment and the optimality notion under full reward feedback. We now consider a setting where two agents interact with the same environment.

\textbf{The Learner.}
The learner has full access to rewards. At each round $t=1,2,...,N$, the learner observes $(\mathcal{A}_t, X_t)$, chooses $\hat a_t\in \mathcal A_t$ and then observes the reward $r_t$. The learner may use any adaptive contextual bandit algorithm and  $\hat a_t$ may depend on the full interaction history. 

\textbf{The Observer.}
In contrast, the observer operates in a reward-free manner. It is an external party that monitors the learner's behavior. Specifically, the observer has access to the sequence of context-action pairs generated by the learner, i.e., it observes $(\mathcal A_t, X_t, \hat a_t)$ but does not observe the reward $r_t$. 
Unlike the learner, the observer is a passive entity during the $N$ rounds of interaction and does not incur regret directly. The observer tries to recover the environment's parameter, i.e., to estimate a vector $\tilde \theta$ that aligns with the true parameter $\theta^*$, which is crucial for understanding why the learner favors certain actions than others. Our approach quantifies such recovery using Predictive Regret:
\begin{definition}[Predictive Regret]\label{def:singlestepregret}
    For any deterministic policy $\pi_\theta$ mapping $(\mathcal{A}_t, X_t)$ to an arm in $\mathcal{A}_t$, define its predictive regret:
    \begin{align*}
        \rho(\pi_\theta):=\mathbb E_{(\mathcal{A}_t, X_t)\sim D} \left[ \langle x_{a^*_t}, \theta^* \rangle - \langle x_{\pi_\theta(\mathcal{A}_t, X_t)}, \theta^* \rangle \right].
    \end{align*}
\end{definition}
The observer does not interact with the environment. The predictive regret is a theoretical construct serving as a proxy for parameter identification. A value of small $\rho(\pi_{\tilde\theta})$ implies that $\tilde{\theta}$ successfully captures the decision boundaries of $\theta^*$ relevant to the distribution $\mathcal{D}$.

While the ultimate goal of the observer is to recover the true environment parameter $\theta^*$, the problem is fundamentally ill-posed due to the inherent identifiability issues. Without access to realized rewards, distinguishing between parameters that induce identical rankings of arms is often impossible. Indeed, a whole equivalence class of parameters may rationalize the learner's behavior equally well. Consequently, rather than demanding the recovery of the exact ground truth, we adopt the predictive regret as a more pragmatic objective only to accurately capture the underlying decision boundaries and functionally reproduces the optimal policy.

\section{Two-Phase Suffix Imitation}

The observer receives the full sequence of the learner's interaction history $D_{\mathrm{obs}}^{1:N}=\{(\mathcal{A}_t, X_t,\hat a_t)\}_{t=1}^N$. A naive approach would be to utilize all available data for training. However, this ignores the non-stationary quality of the learner's behavior: in the early stages, the learner is primarily exploring and its actions are often suboptimal, whereas in later stages, the learner converges towards the optimal policy. To address this, we propose a simple yet effective strategy, naming it Two-Phase Imitation strategy: the observer strategically discards the early portion of the history and leverages the high-quality trajectory for learning. Despite its simplicity, we show in Theorem \ref{thm:final_regret} that this approach is sufficient to achieve strong theoretical guarantees.

Formally, we partition the total horizon $N$ into two phases based on a dynamic cutoff point $T(N)$:

\textbf{Phase \Rmnum{1}: Burn-In.} For rounds $t=1,2,...,T(N)$, the observer treats the data $D_{\mathrm{obs}}^{1:T(N)}$ as unreliable samples and excludes them from the training set. During this phase, the learner executes a standard online no-regret algorithm. By exploiting reward feedback, the learner improves its policy, gradually converging toward the optimal arm $a_t^*$. The length $T(N)$ is a parameter dependent on the horizon $N$.

\textbf{Phase \Rmnum{2}: Imitation.} 
In this phase, the learner is assumed to have achieved a sufficient level of accuracy(satisfying Assumption \ref{ass:massartnoise}). The observer collects the dataset during rounds $t=T(N)+1,...,N$, the observer collects the observational dataset $D_\text{obs}^{T(N)+1,N}$, and uses it to construct a reward-free policy $\tilde \pi$.
By focusing exclusively on Phase \Rmnum{2}, we effectively transform the complex problem of learning from a non-stationary agent into a cleaner problem of learning from noisy labels with bounded error rates. 

It is counterintuitive that the observer can achieve optimality under a strict information deficit, having never observed a single reward. However, we show that the reward signal becomes redundant in the limit, as the learner’s converged actions effectively encode the underlying optimal decision boundaries (Theorem~\ref{thm:final_regret}).

\subsection{Learner's Algorithm}

Instead of restricting the learner to a specific algorithm, we consider a general class of adaptive strategies $\hat{\Pi}$. 
To formalize the learner's reliability, we draw inspiration from the Massart noise condition \cite{massart2006risk} in statistical learning, which requires label noise to be bounded away from $1/2$. Adapting this to our non-stationary setting where the agent improves over time, we introduce a Dynamic Massart Noise condition:
\begin{assumption}
\label{ass:massartnoise}
We assume the learner's performance improves over time. Specifically, there exists a non-increasing function 
$\eta: \mathbb{N} \mapsto [0, 1)$ 
such that for any burn-in period $T$, the learner's error probability in the subsequent phase ($t > T$) is bounded by:
\begin{align*}
\sup_{t > T} \mathbb{P}(\hat{a}_t \neq a^*_t \mid \mathcal{A}_t, X_t) \le \eta(T).
\end{align*}
We require the burn-in period $T$ to be sufficiently large such that $\eta(T) < 1/2$.
\end{assumption}

%
Under the standard noise assumption on $\xi_t$, we demonstrate that this assumption is satisfied by a broad class of standard contextual bandit algorithms. Consider any learner that achieves sublinear cumulative regret, a standard requirement for rational agents. Suppose the learner's policy is a standard no-regret algorithm. 
Since each mistake incurs a regret of at least $\Delta_{\min}$, the expected number of mistakes up to horizon $N$ is bounded, thus resulting in a bounded average error probability. This implies that standard no-regret algorithms naturally satisfy Assumption \ref{ass:massartnoise}. Prominent examples include LinUCB and LinTS.
\begin{example}[LinUCB]
    LinUCB \cite{li2010contextual} maintains an upper confidence bound for the reward of each arm. It is well-known to achieve a regret bound of $\tilde{O}(\sqrt{N})$.
\end{example}
\begin{example}[LinTS]
    Thompson Sampling also achieves $\tilde{O}(\sqrt{N})$ regret guarantees \cite{agrawal2013thompson}. Consequently, it also complies with condition \ref{ass:massartnoise}, providing a vanishing error rate as the burn-in period $T$ increases.
\end{example}

The assumption is not an artificial constraint but a direct consequence of the learner's learning phase. As long as the learner effectively explores and exploits, the observer can rely on the increasingly accurate action trace in Phase \Rmnum{2}.

\subsection{Observer's Algorithm}\label{observer}
\begin{algorithm}[tb]
  \caption{Suffix Imitation via ERM}
  \label{alg:observer}
  \begin{algorithmic}
    \STATE {\bfseries Input:} Total horizon $N$, Burn-in period $T(N)$, Phase-\Rmnum{2} trace $D_{\mathrm{obs}}^{T(N)+1:N}$.
    \STATE {\bfseries Output:} Learned policy $\tilde{\pi}$.
    \STATE Set $L(N)\leftarrow N-T(N)$.
    \STATE Compute $\tilde\theta$ by minimizing the ERM:
    \[
      \tilde{\theta}\in\argmin_{\theta\in\mathbb R^d}\frac{1}{L(N)}\sum_{t=T(N)+1}^{N}\mathbf 1\!\left[\pi_\theta(\mathcal A_t,X_t)\neq \hat a_t\right].
    \]
    \STATE \textbf{return} $\tilde{\pi}(\cdot) = \argmax_{a} \langle x_a, \tilde{\theta} \rangle$.
  \end{algorithmic}
\end{algorithm}
During phase \Rmnum{2}, the observer passively collects the learner's context-action trace $D_\mathrm{obs}^{T(N)+1:N}=\{(\mathcal A_t, X_t, \hat a_t)\}_{t=T(N)+1}^N$, where the ``label'' $\hat a_t$ is the learner's chosen arm. Motivated by the post burn-in bounded-noise model $\Pr(\hat a_t \neq a_t^*|\mathcal{A}_t, X_t)\leq \eta(T)$, we treat $\hat a_t$ as a noisy proxy of the optimal arm and learn a policy through Empirical Risk Minimization (ERM).

Concretely, we consider a linear scoring policy parameterized by a shared vector $\theta\in \mathbb R^d$:
$\pi_\theta(\mathcal A_t, X_t):=\argmax _{a\in\mathcal A_t} \langle x_a, \theta \rangle$. The observer fits $\theta$ by minimizing the empirical $0-1$ imitation loss on $D_\mathrm{obs}^{T(N)+1:N}$. Define $L(N)=N-T(N)$, the loss is:
\[
    \tilde \theta \in \argmin_\theta \frac{1}{L(N)} \sum_{t=T(N)+1}^N \mathbf{1}[\pi_\theta(\mathcal A_t, X_t)\neq \hat a_t], 
\]
where $\hat a_t$ is the learner's chosen arm. $\tilde \pi:=\pi_{\tilde\theta}$.
Finally, the observer achieves a vanishing predictive regret $\rho(\tilde\pi)$.

The algorithm demonstrates that recovering the optimal policy does not require observing outcomes. This shifts the core challenge from learning from feedback to interpreting revealed preferences, proving that the context-action trace alone suffices for identification.
\section{Provable Guarantees for the Observer’s Learned Policy}
In this section, we provide a finite-sample analysis of the observer's policy learned via ERM on the suffix dataset $D_\mathrm{obs}^{T(N)+1:N}$. Our goal is to bound the predictive regret (Definition~\ref{def:singlestepregret}) of the learned policy $\tilde\pi$ with respect to the observer's effective sample size $L(N)=N-T(N)$.

{\bfseries From regret to decision errors.}
We begin by linking regret to the event that a policy selects a suboptimal arm. 
For any policy $\pi$, let $a_t^\star \in \argmax_{a\in\mathcal{A}_t}\langle x_a,\theta^\star\rangle$ denote an optimal arm at round $t$. 
The instantaneous regret incurred by $\pi$ at round $t$ is
$\Delta_t(\pi) := \big\langle x_{a_t^\star},\theta^\star\big\rangle -\big\langle x_{\pi(\mathcal{A}_t,X_t)},\theta^\star\big\rangle$.
The following lemma shows that bounded features and parameters imply a uniform bound on the reward gaps, which in turn caps the per-round regret. Let $\mathcal{D}$ denote the unknown distribution that generates $(\mathcal{A}_t,X_t)$ in this phase.

\begin{lemma}
\label{lem:uniformgap}
Assume $\|x_a\|_2\le 1$ for all $a,t$ and $\|\theta^\star\|_2\le 1$. Then for any policy $\pi$, the predictive regret is bounded by clean risk:
\begin{align*} \rho(\pi) \le 2 \cdot \Pr_{(\mathcal{A}_t, X_t) \sim \mathcal{D}}(\pi(\mathcal{A}_t, X_t) \ne a^*).
\end{align*}
\end{lemma}
\begin{proof}
Since reward gaps are bounded by $$\| \langle x_{a^*_t} - x_a , \theta^* \rangle\|_2 \le \| \langle x_{a^*_t}, \theta^* \rangle\|_2 + \| \langle x_a, \theta^* \rangle\|_2 \le 2,$$ the expected regret is at most 2 times the probability of choosing a suboptimal arm.
\end{proof}

{\bfseries Connecting Clean and Noisy Risks.}
In Phase~\Rmnum{2}, the observer collects labeled observations $D_{\mathrm{obs}}^{T(N)+1:N} =\{(\mathcal A_t, X_t, \hat a_t)\}_{t=T(N)+1}^N $. We measure the observer's performance using two risks: a \emph{clean} risk against the optimal arm and a \emph{noisy imitation} risk against the learner's action,
\begin{align*}
    \mathcal{R}(\pi)
    &:=\Pr_{(\mathcal{A}_t,X_t)\sim \mathcal{D}}\!\left[\pi(\mathcal{A}_t,X_t)\neq a^*_t\right],\\
    \mathcal{R}_{\mathrm{noisy}}(\pi)
    &:=\Pr_{(\mathcal{A}_t,X_t)\sim \mathcal{D}}\!\left[\pi(\mathcal{A}_t,X_t)\neq \hat a_t\right],
\end{align*}
where $a^*_t$ is an optimal arm, and
$\hat a_t$ denotes the learner's noisy choice under context $(\mathcal{A}_t,X_t)$. Intuitively, $\mathcal{R}(\pi)$ captures the probability of choosing a suboptimal arm, while $\mathcal{R}_{\mathrm{noisy}}(\pi)$ quantifies how well the observer imitates the learner.
The key step is to relate these two quantities under a Massart type noise condition.

\begin{lemma}\label{lem:massart_transfer}
Let $\pi^\star(\mathcal{A}_t,X_t)=a^*_t$.
Under Assumption~\ref{ass:massartnoise} (i.e., $\Pr[\hat a
_t\neq a^*_t\mid \mathcal{A}_t,X_t] \leq \eta(T)$ with $\eta(T)<1/2$),
for any policy $\pi$,
\[
\mathcal{R}_{\mathrm{noisy}}(\pi)-\mathcal{R}_{\mathrm{noisy}}(\pi^\star)
\;\ge\;
(1-2\eta(T))\,\mathcal{R}(\pi).
\]
\end{lemma}

\begin{proof}
    Fix $(\mathcal A_t,X_t)$. Write $\eta_t=\Pr(\hat a_t\neq a_t^*\mid \mathcal A_t,X_t)\le \eta(T)$. $\pi^*(\mathcal A_t,X_t)=a_t^*$. Define the conditional noisy loss $$\ell_{\mathrm{noisy}}(\pi;\mathcal A_t,X_t)\;:=\;\Pr\!\big(\pi(\mathcal A_t,X_t)\neq \hat a_t\mid \mathcal A_t,X_t\big).$$
    If $\pi(\mathcal A_t,X_t)=a_t^*$, then $$\ell_{\mathrm{noisy}}(\pi;\mathcal A_t,X_t)=\ell_{\mathrm{noisy}}(\pi^*;\mathcal A_t,X_t),$$ and the conditional excess is $0$. Otherwise let $b:=\pi(\mathcal A_t,X_t)\neq a_t^*$. Then
    \begin{align*}
        &\ell_{\mathrm{noisy}}(\pi;\mathcal A_t,X_t)-\ell_{\mathrm{noisy}}(\pi^*;\mathcal A_t,X_t)\\
        =&\Pr(\hat a_t=a_t^*\mid \mathcal A_t,X_t)-\Pr(\hat a_t=b\mid \mathcal A_t,X_t).
    \end{align*}
    Moreover, $\Pr(\hat a_t=b\mid \mathcal A_t,X_t)\le \Pr(\hat a_t\neq a_t^*\mid \mathcal A_t,X_t)=\eta_t$, hence the above is at least $(1-\eta_t)-\eta_t=1-2\eta_t$. Therefore, for all $(\mathcal A_t,X_t)$,
    \begin{align*}
        &\ell_{\mathrm{noisy}}(\pi;\mathcal A_t,X_t)-\ell_{\mathrm{noisy}}(\pi^*;\mathcal A_t,X_t)\\
        \ge&
        \begin{cases}
        0& \text{if }\pi(\mathcal A_t,X_t)=a_t^*\\
        1-2\eta_t & \text{if }\pi(\mathcal A_t,X_t)\ne a_t^*
        \end{cases}\\
        =&(1-2\eta_t)\mathbf 1\!\big[\pi(\mathcal A_t,X_t)\neq a_t^*\big].
    \end{align*}
    Taking expectation over $(\mathcal A_t,X_t,\hat a_t)\sim D_{\mathrm{obs}}$ and using $\eta_t\le \eta(T)$ yields
    $$
    \mathcal R_{\mathrm{noisy}}(\pi)-\mathcal R_{\mathrm{noisy}}(\pi^*)\ \ge\ (1-2\eta(T))\,\mathcal R(\pi),
    $$
    as claimed.
\end{proof}

Under the conditions of Lemma~\ref{lem:massart_transfer}, for any policy $\pi$,
\[
\mathcal{R}(\pi)
\;\le\;
\frac{\mathcal{R}_{\mathrm{noisy}}(\pi)-\mathcal{R}_{\mathrm{noisy}}(\pi^\star)}{1-2\eta(T)}.
\]
It suffices to control the observer's noisy imitation risk (relative to $\pi^\star$) to obtain a bound on the clean error rate against the optimal arm.

{\bfseries ERM Generalization Under Bounded Noise}
Recall the observer's policy class $\Pi:=\Big\{\pi_\theta(\mathcal A_t,X_t):=\argmax_{a\in\mathcal A_t}\langle x_a,\theta\rangle \;:\; \theta\in\mathbb R^d\Big\}$.
In Phase~\Rmnum{2}, the observer fits $\tilde\theta$ by ERM using the learner's chosen arms as imitation labels.
Define the empirical noisy imitation risk as 
\[
\tilde{\mathcal R}_{\mathrm{noisy}}(\pi):=\frac{1}{L
(N)}\sum_{t=T+1}^{N}\mathbf{1}\big[\pi(\mathcal{A}_t,X_t)\neq \hat a_t\big],
\]
where $L(N)=N-T(N)$ is the horizon of phase \Rmnum{2} and $\tilde\pi\in\argmin_{\pi\in\Pi}\tilde{\mathcal R}_{\mathrm{noisy}}(\pi)$.
To control the generalization gap between $\tilde{\mathcal R}_{\mathrm{noisy}}$ and the population noisy risk $\mathcal R_{\mathrm{noisy}}(\pi)$,
we use the Natarajan dimension \cite{natarajan1989learning}, a standard complexity measure for multiclass $0$-$1$ learning.

\begin{definition}[Natarajan Dimension] 
A set of instances $S=\{(\mathcal A_1,X_1),\ldots,(\mathcal A_m,X_m)\}$ is N-shattered by policy class $\Pi$ if there exist two distinct actions
$a_{i,0},a_{i,1}\in\mathcal A_i$ for each $i\in[m]$ such that for every $ v\in\{0,1\}^m$, there exists $\theta_{ v}\in\mathbb R^d$ satisfying
\[
\argmax_{a\in\mathcal A_i}\langle x_{a},\theta_{ v}\rangle=a_{i,v_i}, \forall i\in[m].
\]
The Natarajan dimension of $\Pi$, denoted $d_{\mathrm{Nat}}(\Pi)$, is the largest $m$ for which such a set can be N-shattered.
\end{definition}

Recall that our observer minimizes loss over the class of linear scoring policies $\Pi$ defined before in section~\ref{observer}. We now bound the complexity of this class.
\begin{lemma}
\label{lem:complexity}
Consider the class of linear policies $\Pi := \{ \pi_\theta(\cdot) = \argmax_{a \in \mathcal{A}} \langle x_a, \theta \rangle \mid \theta \in \mathbb{R}^d \}$. As established by prior works on the Natarajan dimension of linear multiclass predictors \cite{daniely2015multiclass}, the complexity of $\Pi$ with $K$ arms is bounded as:
   \begin{align*}
    d_{\mathrm{Nat}}(\Pi) = O(d \log K).
    \end{align*}
\end{lemma}

\begin{lemma}\label{lem:uniform_convergence} 
Let $\Pi$ be a hypothesis class with Natarajan dimension $d_{Nat}(\Pi)$. For any $\delta \in (0,1)$, with probability at least $1-\delta$ over the generation of the suffix dataset $D_{\text{obs}}^{T+1:N}$ of size $L$, we have:
\begin{align*}
&\sup_{\pi \in \Pi} \left| \mathcal{R}_{\mathrm{noisy}}(\pi) - \tilde{\mathcal{R}}_{\mathrm{noisy}}(\pi) \right| \\
\le &C\cdot \sqrt{\frac{d_{Nat}(\Pi) \log L(N) + \log(1/\delta)}{L(N)}},
\end{align*}
where $C$ is a constant independent of $d$, $K$, and $L(N)$.
\end{lemma}

In the subsequent analysis, we use the constant $C$ to encapsulate the absolute constants arising from the concentration inequalities, allowing us to focus on the convergence rate. While the full derivation of Lemma \ref{lem:uniform_convergence} follows standard concentration techniques, we defer the formal proof to Appendix~\ref{appdix:proof5.5} to stay within space constraints. Specifically, the result follows by combining the growth function bound for classes with finite Natarajan dimension \cite{natarajan1989learning, daniely2015multiclass} with the Azuma-Hoeffding inequality \cite{azuma1967weighted} to handle the adaptive nature of the data generation process. While the standard results are stated for i.i.d. data, the extension to the adaptive setting, where contexts are i.i.d. but labels $\hat{a}_t$ depend on history is straightforward for bounded loss functions. 

{\bfseries Generalization and Regret Bound.}
We now incorporate realizability by assuming there exists $\theta^\star$ such that
$\pi^\star(\mathcal A,X)=\pi_{\theta^\star}(\mathcal A,X)=a^\star(\mathcal A,X)$ almost surely over $(\mathcal A,X)\sim\mathcal D$.
Under post burn-in Massart noise with parameter $\eta(T)<\tfrac12$, the noisy imitation excess risk controls the clean risk $\mathcal R(\pi)$ through the Massart transfer inequality (Lemma  \ref{lem:massart_transfer}), and ERM further converts the population guarantee into a predictive regret bound.

\begin{theorem}[Clean-Risk Guarantee for ERM]\label{thm:erm_clean}
Assume realizable $\pi^\star\in\Pi$ and post burn-in Massart noise with parameter $\eta(T)<\tfrac12$, where $T=T(N)$ is the burn-in length.
Then for any $\delta\in(0,1)$, with probability at least $1-\delta$,
\[
\mathcal R(\tilde\pi)
\le
\frac{C}{1-2\eta(T)}\sqrt{\frac{d\log K\cdot\log L(N)+\log(1/\delta)}{L(N)}},
\]
where $L(N)=N-T(N)$ is the horizon of phase \Rmnum{2}, $C$ is a constant independent of problem parameters.
\end{theorem}

\begin{proof}
Let $L = N - T$ be the horizon of Phase II. We instantiate the uniform convergence bound from Lemma \ref{lem:uniform_convergence}. Define the generalization error bound $\varepsilon_L$ as:
\begin{align*}
    \varepsilon_L := C \sqrt{\frac{d_{\mathrm{Nat}}(\Pi) \log L + \log(1/\delta)}{L}}.
\end{align*}
By Lemma \ref{lem:uniform_convergence}, the event 
$$
\mathcal{E} := \left\{ \sup_{\pi \in \Pi} \left| \mathcal{R}_{\mathrm{noisy}}(\pi) - \tilde{\mathcal{R}}_{\mathrm{noisy}}(\pi) \right| \le \varepsilon_L \right\}
$$
holds with probability at least $1 - \delta$.

Condition on the event $\mathcal{E}$. Since $\tilde{\pi}$ is the empirical risk minimizer on the suffix dataset, we have $\tilde{\mathcal{R}}_{\mathrm{noisy}}(\tilde{\pi}) \le \tilde{\mathcal{R}}_{\mathrm{noisy}}(\pi^\star)$. Therefore:
\begin{align*}
    \mathcal{R}_{\mathrm{noisy}}(\tilde{\pi}) 
    &\le \tilde{\mathcal{R}}_{\mathrm{noisy}}(\tilde{\pi}) + \varepsilon_L && \text{(by event } \mathcal{E} \text{)} \\
    &\le \tilde{\mathcal{R}}_{\mathrm{noisy}}(\pi^\star) + \varepsilon_L && \text{(by optimality of } \tilde{\pi} \text{)} \\
    &\le \mathcal{R}_{\mathrm{noisy}}(\pi^\star) + 2\varepsilon_L && \text{(by event } \mathcal{E} \text{ on } \pi^\star \text{)}.
\end{align*}
This implies the excess noisy risk is bounded by:
\[
\mathcal{R}_{\mathrm{noisy}}(\tilde{\pi}) - \mathcal{R}_{\mathrm{noisy}}(\pi^\star) \le 2\varepsilon_L.
\]
According to Lemma \ref{lem:massart_transfer}, we have $\Pr(\hat a_t\neq a_t^*\mid \mathcal A_t,X_t)\le \eta(T)<\tfrac12$. Thus, 
    $$
    \mathcal R_{\mathrm{noisy}}(\tilde\pi)-\mathcal R_{\mathrm{noisy}}(\pi^*)\ \ge\ (1-2\eta(T))\,R(\tilde\pi).
    $$
    Combining the two displays yields, on $\mathcal E$, $R(\tilde\pi)\ \le\ \frac{2\varepsilon_L}{1-2\eta(T)}$.
    Here $\varepsilon_L$ can be instantiated via Lemma \ref{lem:uniform_convergence} and Lemma \ref{lem:complexity}, which yields:
   \begin{align*}
   \mathcal R(\tilde\pi)&\le\frac{2C}{1-2\eta(T)}\sqrt{\frac{d_{\mathrm{Nat}}(\Pi)\log L(N)+\log(1/\delta)}{L(N)}}\\
   &\le\frac{C'}{1-2\eta(T)}\sqrt{\frac{d\log K\cdot\log L(N)+\log(1/\delta)}{L(N)}},
   \end{align*}
   where $C'$ is another constant.
\end{proof}
Theorem~\ref{thm:erm_clean} captures the intuition that once the learner's post burn-in actions are more often correct than incorrect conditioned on $(\mathcal A,X)$ (i.e., $\eta(T)<1/2$), minimizing the imitation error over $\Pi$ recovers a policy that matches the optimal arm on most contexts.
We now translate the clean decision error bound into a cumulative regret bound. Consider deploying the observer's policy $\tilde \pi$ for horizon $L$ with i.i.d. $D_\mathrm{obs}=\{(\mathcal A_t, X_t)\}_{t=T+1}^N$. 

\begin{theorem}[Predictive Transfer Regret Bound]
\label{thm:final_regret}
Let $T=T(N)$ be the length of Phase I. Under the Dynamic Massart Noise condition (Assumption \ref{ass:massartnoise}), let $\tilde{\pi}$ be the policy learned by ERM on the suffix dataset of length $L(N)$. 
With probability at least $1-\delta$, the predictive transfer regret of $\tilde{\pi}$ satisfies:
\begin{equation*}
    \rho(\tilde{\pi}) \le \frac{C}{1-2\eta(T)} \sqrt{\frac{d \log K \cdot \log L(N) + \log(1/\delta)}{L(N)}},
\end{equation*}
where $C$ is a problem-dependent constant.
\end{theorem}

\begin{proof}
This result is an immediate consequence of Lemma~\ref{lem:uniformgap} and Theorem~\ref{thm:erm_clean}. First, Lemma~\ref{lem:uniformgap} establishes that the predictive regret is bounded by twice the clean classification risk: $\rho(\tilde{\pi}) \le 2 \mathcal{R}(\tilde{\pi})$.
Theorem~\ref{thm:erm_clean} provides a high-probability bound on the clean risk $\mathcal{R}(\tilde{\pi})$:
\begin{align*}
    \mathcal{R}(\tilde{\pi}) \le \frac{C}{1-2\eta(T)} \sqrt{\frac{d \log K \cdot \log L(N) + \log(1/\delta)}{L(N)}}.
\end{align*}
Substituting the bound directly yields the stated result.
\end{proof}

The bound in Theorem~\ref{thm:final_regret} reveals a fundamental trade-off governed by the length of phase \Rmnum{1} period $T$. The regret consists of two competing terms: The first term is the learner's label noise $(1-2\eta(T))^{-1}$. Since $\eta(\cdot)$ is non-increasing (Assumption~\ref{ass:massartnoise}), a larger $T$ reduces the learner's error rate. The second term of generalization error $\tilde O(1/\sqrt{(L(N)})$, which scales with the effective sample size $L(N)=N-T(N)$. A larger T reduces the number of samples available for ERM, thus increasing the error.
Consequently, the optimal choice of $T$ requires balancing the quality of labels against the quantity of the training data.

\begin{corollary}
\label{cor:conservative_burnin}
Suppose the learner's strategy satisfies a cumulative regret bound of $\mathbb{E}[R_T(\hat{\pi})] \le C_{\text{alg}} \sqrt{T}$. Further assume a unique optimal arm and a uniform gap $\Delta_{\min}>0$ on $\mathrm{supp}(\mathcal{D})$ (Definition~\ref{def:regretgap}). 
If we employ a conservative burn-in period of $T(N) = \Theta(N^\alpha)$ with $\alpha\in(0,1)$, then for sufficiently large $N$, the condition $\eta(T) < 1/2$ holds, and the observer's predictive regret scales as:
\begin{equation*}
    \rho(\tilde{\pi}) \le \tilde{O}\left( \sqrt{\frac{d \log K}{N}} \right).
\end{equation*}
\end{corollary}

\begin{proof}
    Recall that each suboptimal action incurs an instantaneous regret of at least $\Delta_{\min}$ (Definition~\ref{def:regretgap}). 
    Therefore, the expected number of mistakes up to time $T$ is bounded by $\mathbb{E}[R_T] / \Delta_{\min}$. 
    Since the error rate $\eta(t)$ is non-increasing (Assumption~\ref{ass:massartnoise}), the instantaneous error rate at time $T$ is bounded by the average error rate:
    \begin{equation*}
        \eta(T) \le \frac{1}{T} \sum_{t=1}^T \mathbb{P}(\hat{a}_t \neq a^*_t) \le \frac{\mathbb{E}[R_T(\hat{\pi})]}{T \Delta_{\min}} \le \frac{C_{\text{alg}}}{\Delta_{\min} \sqrt{T}}.
    \end{equation*}
    Next, with the conservative burn-in period $T(N) = \Theta(N^{\alpha})$, we have $\lim_{N \to \infty} T(N) = \infty$. 
    Consequently, $\lim_{N \to \infty} \eta(T(N)) = 0$. 
    Thus, for sufficiently large $N$, the condition $\eta(T) < 1/2$ required by Theorem~\ref{thm:final_regret} is satisfied, and the noise inflation factor $(1 - 2\eta(T))^{-1}$ approaches $1$.
    
    Consider the effective sample size $L(N) = N - T(N) = N - \Theta(N^{\alpha})$. 
    Since the exponent $\alpha < 1$, the term $N$ dominates, and $L(N) = \Theta(N)$. 
    Substituting these into the bound from Theorem~\ref{thm:final_regret} and noting that $\sqrt{T} = \Theta(N^{\alpha/2})$:
    \begin{align*}
        \rho(\tilde{\pi}) 
        &\le \frac{C'\Delta_{\min}\Theta(N^{\alpha/2})}{\Delta_{\min}\Theta(N^{\alpha/2})-2C_\text{alg}} \cdot \\
        &\sqrt{\frac{d \log K \cdot \log(N - \Theta(N^{\alpha})) + \log(1/\delta)}{N - \Theta(N^{\alpha})}} \\
        &=O(1) \cdot \tilde{O}\left(\sqrt{\frac{d \log K}{N}}\right) \\
        &= \tilde{O}\left(\sqrt{\frac{d \log K}{N}}\right).
    \end{align*}
    As $N \to \infty$, the term $\Theta(N^{\alpha/2})$ dominates the constant $2C_{\text{alg}}$, making the coefficient approach 1. Similarly, $N$ dominates $\Theta(N^{\alpha})$ in the denominator, recovering the standard convergence rate.
\end{proof}
\begin{remark}
It is instructive to directly compare the observer's performance with that of the learner. A standard no-regret learner (e.g., LinUCB or LinTS) typically achieves \emph{an average regret rate} of $\tilde O(1/\sqrt{N})$ \cite{chu2011contextual, agrawal2013thompson}. 
This corollary establishes that our observer attains the same asymptotic rate of $\tilde{O}(1/\sqrt{N})$ in predictive regret. 
This result is particularly striking given the strict information asymmetry: while the learner relies on full reward feedback to converge, the observer matches this efficiency using significantly less information.
Furthermore, the choice of $\alpha$ demonstrates that we do not need to precisely tune the burn-in length. Even with a heavily biased choice that prioritizes label quality over sample quantity, the observer's convergence rate is dominated by the linear growth of the remaining horizon ensuring optimal efficiency. It implies that \emph{even with limited information, the observer still identify the optimal decision boundaries.}
\end{remark}
\section{Experiments}

\begin{figure*}[t] 
    \centering
    
    \begin{minipage}{0.49\textwidth}
        \centering
        \begin{subfigure}[b]{0.48\linewidth}
            \includegraphics[width=\linewidth]{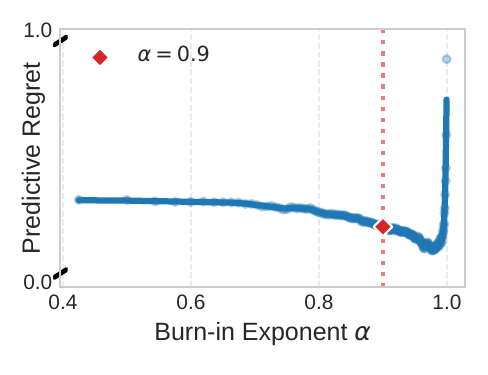}
            \caption{LinTS-Regret}
            \label{fig:regret_ts_row}
        \end{subfigure}
        \hfill
        \begin{subfigure}[b]{0.48\linewidth}
            \includegraphics[width=\linewidth]{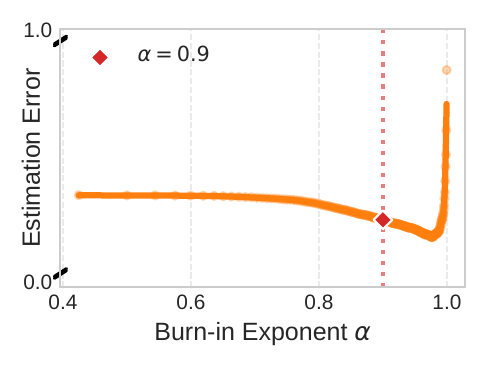}
            \caption{LinTS-Error}
            \label{fig:error_ts_row}
        \end{subfigure}
    \end{minipage}
    \hfill 
    \begin{minipage}{0.49\textwidth}
        \centering
        
        \begin{subfigure}[b]{0.48\linewidth}
            \includegraphics[width=\linewidth]{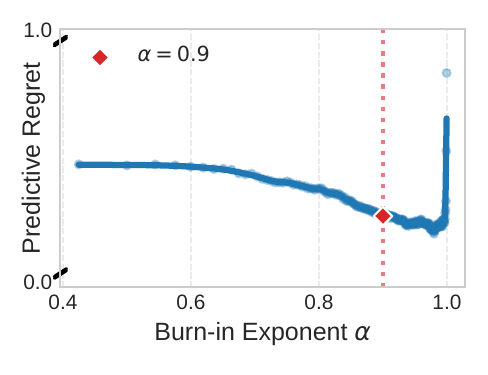}
            \caption{LinUCB-Regret}
            \label{fig:regret_ucb_row}
        \end{subfigure}
        \hfill
        \begin{subfigure}[b]{0.48\linewidth}
            \includegraphics[width=\linewidth]{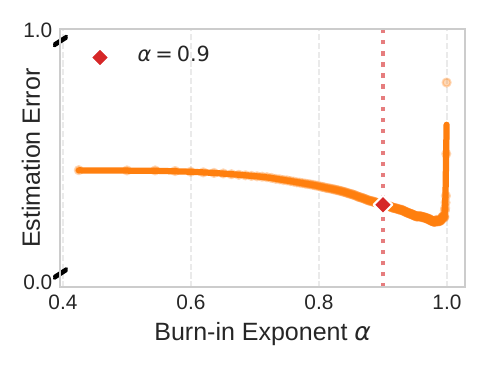}
            \caption{LinUCB-Error}
            \label{fig:error_ucb_row}
        \end{subfigure}
    \end{minipage}

    \caption{Performance comparison with full horizon layout ($d=50$, $K$=200). The left block illustrates LinTS performance, while the right block displays LinUCB results.}
    \label{fig:comparison_1x4}
\end{figure*}
We evaluate the proposed suffix imitation framework on a linear contextual bandit environment. Our primary objective is to empirically validate the trade-off induced by the two-phase strategy and to verify the effectiveness of our theoretically motivated schedules. Furthermore, we aim to demonstrate that a passive observer, through selective hindsight, can achieve a similar parameter identification compared to the active learner.

\subsection{Experimental Setup}
{\bfseries Learner and Observer.}
The learner is an online linear bandit agent with full reward feedback. We report results for learner using a LinUCB and LinTS algorithm separately. The observer is reward-free and only observes $(\mathcal A_t, X_t, \hat a_t)$. 

In phase \Rmnum{2}, it trains a linear scoring policy
$\pi_\theta(\mathcal A_t, X_t)=\argmax_{a\in\mathcal A_t}\langle x_{a},\theta \rangle$ by minimizing the conditional softmax loss over the candidate set $\mathcal A_t$,
as a convex surrogate for the $0$--$1$ imitation objective.\footnote{
Softmax cross-entropy is classification-calibrated: excess surrogate risk controls excess $0$--$1$ risk; see
\citet{zhang2004statistical}. Multiclass consistency is discussed in \citet{tewari2007consistency}.
}
We compare: (\rmnum{1}) naive imitation using all rounds(i.e., burn-in length $T(N)=0$); (\rmnum{2}) an oracle burn-in that selects $T$ in hindsight by sweeping $T$ and choosing the best; and (\rmnum{3}) a rule-based phase \Rmnum{1} length $T(N)=N^{0.9}$ which is our implementable choice, reported alongside the oracle.

{\bfseries Metrics.}
We evaluate the observer's predictive regret $\rho(\tilde \pi)=\mathbb E \left[ \langle x_{a^*_t}, \theta^* \rangle - \langle x_{\pi_\theta(\mathcal{A}_t, X_t)}, \theta^* \rangle \right]$, estimated on an i.i.d. test set using noiseless expected rewards $\langle x_a, \theta^*\rangle$. Besides, we evaluate the parameter recovery via normalized parameter direction error $\big\|\frac{\tilde\theta}{\|\tilde\theta\|} -\frac{\theta^*}{\|\theta^*\|} \big\|_2$. All results are averaged over several random seeds.
\subsection{Results}

\begin{figure*}[t] 
  \centering
  
  \begin{subfigure}[b]{0.48\textwidth} 
    \centering
    \includegraphics[width=\linewidth]{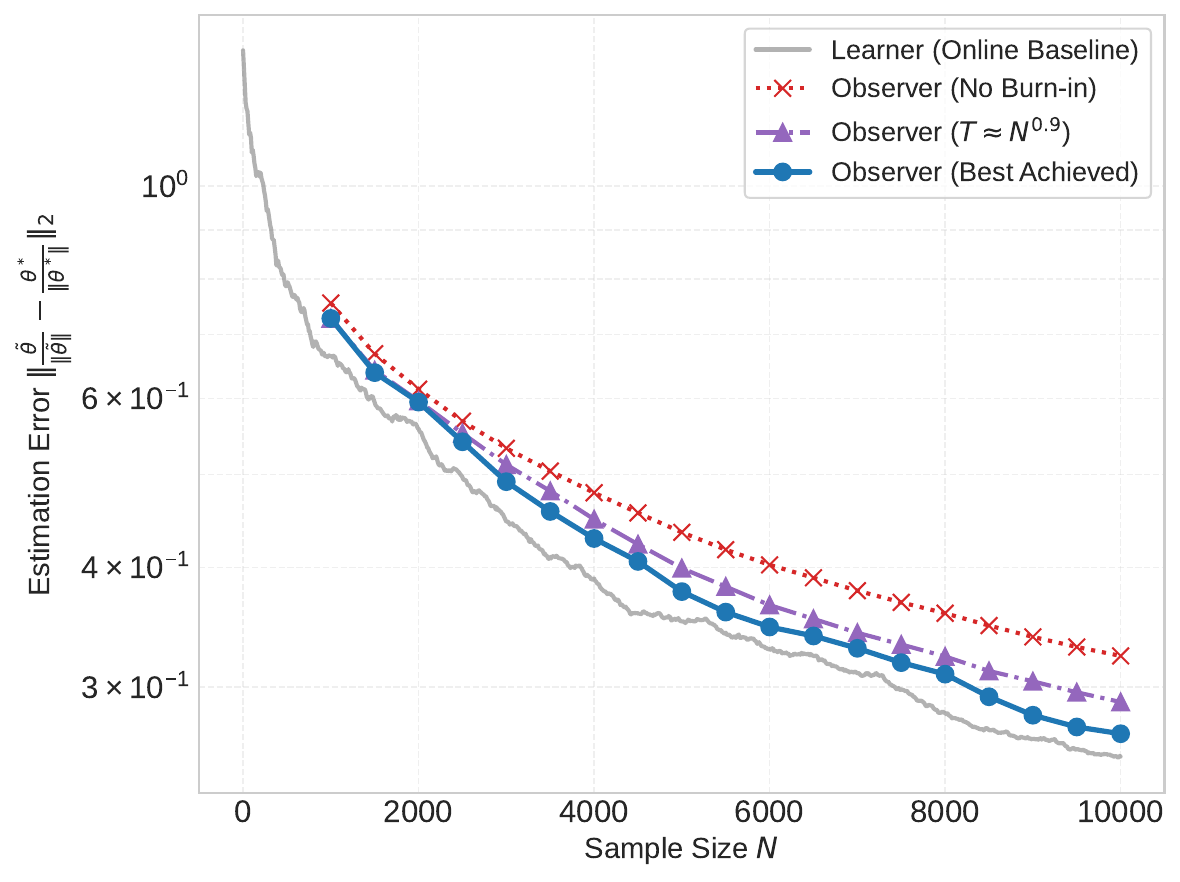}
    \caption{LinTS Comparison}
    \label{fig:comparison_ts}
  \end{subfigure}
  \hfill 
  \begin{subfigure}[b]{0.48\textwidth}
    \centering
    \includegraphics[width=\linewidth]{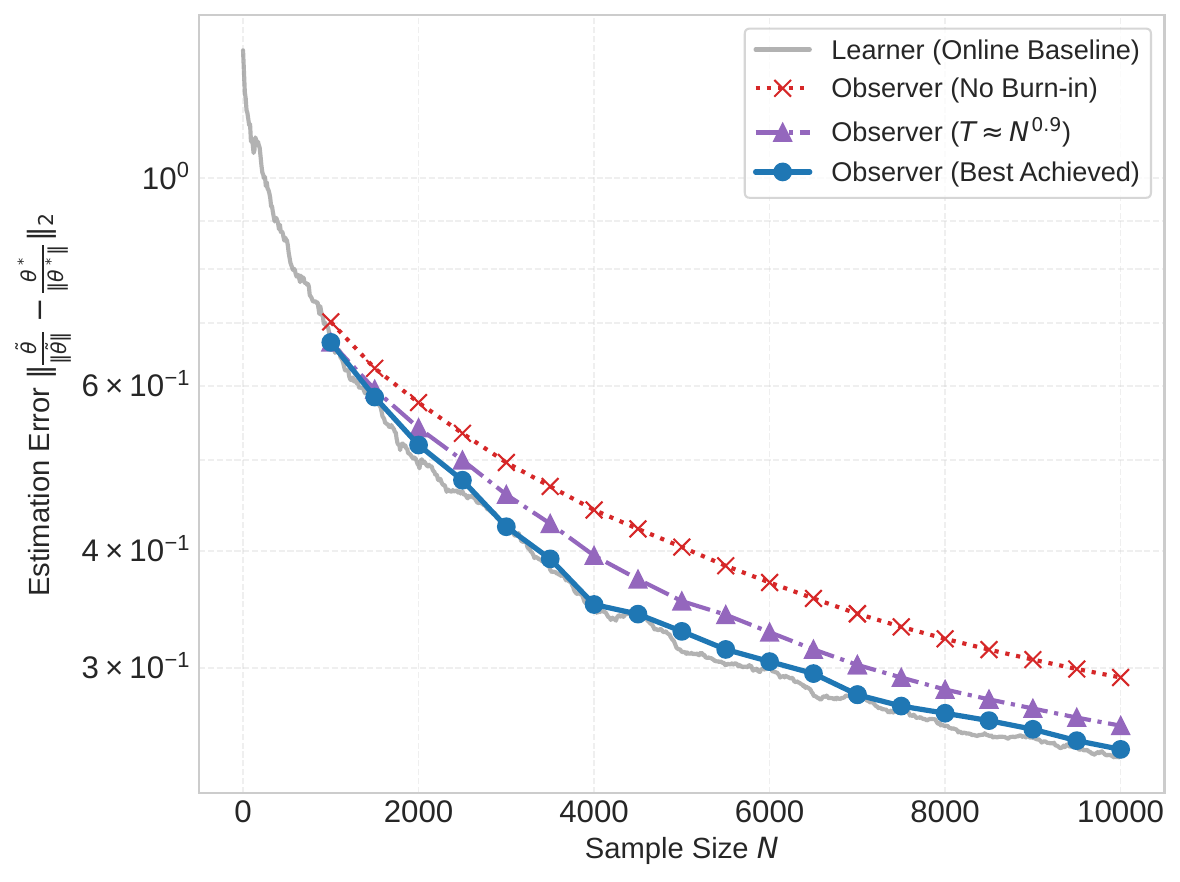}
    \caption{LinUCB Comparison}
    \label{fig:comparison_ucb}
  \end{subfigure}

  \caption{Performance comparison of Observer strategies against the Learner baseline ($d=50$, $K$=200). Figure (a) shows the results for LinTS, and Figure (b) for LinUCB. Both metrics demonstrate that the Observer (Best Achieved) outperforms the online Learner.}
  \label{fig:full_comparison}
\end{figure*}
\paragraph{Impact of Burn-in Length.}
We first validate the necessity of phase \Rmnum{1} by examining the observer's performance across varying burn-in exponent parameter $\alpha$ ($T=N^\alpha$). As shown in Figure \ref{fig:comparison_1x4}, both the predictive decision loss and parameter estimation error exhibit a characteristic U-shaped pattern. Due to space constraints, we present the results for the setting $d=50, K=200$ in the main text. We conducted extensive experiments across various dimensions ($d \in \{20, 50\}$) and arm sets ($K \in \{50, 100, 200\}$). The results across these settings exhibit consistent trends and robust performance, validating our theoretical claims. Detailed plots for these additional configurations are provided in Appendix \ref{appdix:additionalresults}. 

With no burn-in phase (when $T\to 0$), the observer suffers from significant error. This confirms our theoretical concern that the early history of the learner is dominated by exploratory actions, which introduces a heavy distribution shift that confuses the observer. Conversely, when $T\to N$, discarding too much data leads to a sharp increase in error due to insufficient sample size. The optimal performance is consistently achieved at an intermediate ratio, where the observer effectively balances the trade-off between removing biased data and retaining sufficient samples for convergence. This empirically validates Corollary \ref{cor:conservative_burnin}, suggesting that a wide range of polynomial burn-in schedules can mitigate the transfer regret.

\paragraph{Asymptotic Convergence and Interpretability.}
Figure \ref{fig:full_comparison} compares the parameter estimation error of different strategies as the horizon $N$ increases. Several key observations highlight the efficacy of our framework. The Naive Imitation consistently underperforms the Learner. This indicates that simply cloning the learner's full trajectory fails to recover the true parameters accurately, as the dataset is polluted by suboptimal exploratory actions. Remarkably, the oracle observer  achieves an estimation error similar to the Learner itself. This counterintuitive result demonstrates the advantage of the strategy: by selectively learning only from the learner's stationary and optimal phase, the observer can filter out noise and identify the underlying model $\theta^*$ more precisely. 

The rule-based strategy initially suffers from high variance due to aggressive data pruning. However, as $N$ grows large, it rapidly converges, significantly outperforming the Naive baseline and approaching the oracle's performance. This confirms that aggressive burn-in schedules are viable and effective in the large-sample regime.

\section{Conclusion}
We have presented the Inverse Contextual Bandit framework to address the challenge of learning from non-stationary, noisy supervision. Our results establish that information asymmetry does not prevent optimal identification. Theoretically, we show that the underlying utility function is entirely recoverable from interaction traces, which counterintuitively implies that learning agents can be interpreted without reward feedback. These findings are supported by extensive experiments. 
Promising future directions include extending this framework to nonlinear functions and partial monitoring scenarios.

\bibliographystyle{plainnat}
\bibliography{ref}

@inproceedings{li2010contextual,
  title={A contextual-bandit approach to personalized news article recommendation},
  author={Li, Lihong and Chu, Wei and Langford, John and Schapire, Robert E},
  booktitle={Proceedings of the 19th international conference on World wide web},
  pages={661--670},
  year={2010}
}

@article{contextualforhealth,
  title={A contextual-bandit-based approach for informed decision-making in clinical trials},
  author={Varatharajah, Yogatheesan and Berry, Brent},
  journal={Life},
  volume={12},
  number={8},
  pages={1277},
  year={2022},
  publisher={MDPI}
}

@inproceedings{contextforrecommandation,
  title={Neural contextual bandits for personalized recommendation},
  author={Ban, Yikun and Qi, Yunzhe and He, Jingrui},
  booktitle={Companion Proceedings of the ACM Web Conference 2024},
  pages={1246--1249},
  year={2024}
}

@article{daniely2015multiclass,
  title={Multiclass learnability and the ERM principle.},
  author={Daniely, Amit and Sabato, Sivan and Ben-David, Shai and Shalev-Shwartz, Shai},
  journal={J. Mach. Learn. Res.},
  volume={16},
  number={1},
  pages={2377--2404},
  year={2015}
}

@article{azuma1967weighted,
  title={Weighted sums of certain dependent random variables},
  author={Azuma, Kazuoki},
  journal={Tohoku Mathematical Journal, Second Series},
  volume={19},
  number={3},
  pages={357--367},
  year={1967},
  publisher={Mathematical Institute, Tohoku University}
}

@inproceedings{chu2011contextual,
  title={Contextual bandits with linear payoff functions},
  author={Chu, Wei and Li, Lihong and Reyzin, Lev and Schapire, Robert},
  booktitle={Proceedings of the fourteenth international conference on artificial intelligence and statistics},
  pages={208--214},
  year={2011},
  organization={JMLR Workshop and Conference Proceedings}
}

@inproceedings{agrawal2013thompson,
  title={Thompson sampling for contextual bandits with linear payoffs},
  author={Agrawal, Shipra and Goyal, Navin},
  booktitle={International conference on machine learning},
  pages={127--135},
  year={2013},
  organization={PMLR}
}

@article{abbasi2011improved,
  title={Improved algorithms for linear stochastic bandits},
  author={Abbasi-Yadkori, Yasin and P{\'a}l, D{\'a}vid and Szepesv{\'a}ri, Csaba},
  journal={Advances in neural information processing systems},
  volume={24},
  year={2011}
}

@inproceedings{Lattimore2020BanditA,
  title={Bandit Algorithms},
  author={Tor Lattimore and Csaba Szepesvari},
  year={2020},
  url={https://api.semanticscholar.org/CorpusID:242172176}
}

@inproceedings{ng2000algorithms,
  title={Algorithms for inverse reinforcement learning.},
  author={Ng, Andrew Y and Russell, Stuart and others},
  booktitle={Icml},
  volume={1},
  number={2},
  pages={2},
  year={2000}
}

@inproceedings{abbeel2004apprenticeship,
  title={Apprenticeship learning via inverse reinforcement learning},
  author={Abbeel, Pieter and Ng, Andrew Y},
  booktitle={Proceedings of the twenty-first international conference on Machine learning},
  pages={1},
  year={2004}
}

@article{zare2024survey,
  title={A survey of imitation learning: Algorithms, recent developments, and challenges},
  author={Zare, Maryam and Kebria, Parham M and Khosravi, Abbas and Nahavandi, Saeid},
  journal={IEEE Transactions on Cybernetics},
  year={2024},
  publisher={IEEE}
}

@article{zeng2023demonstrations,
  title={When demonstrations meet generative world models: A maximum likelihood framework for offline inverse reinforcement learning},
  author={Zeng, Siliang and Li, Chenliang and Garcia, Alfredo and Hong, Mingyi},
  journal={Advances in Neural Information Processing Systems},
  volume={36},
  pages={65531--65565},
  year={2023}
}

@article{seo2024mitigating,
  title={Mitigating covariate shift in behavioral cloning via robust stationary distribution correction},
  author={Seo, Seokin and Lee, Byung-Jun and Lee, Jongmin and Hwang, HyeongJoo and Yang, Hongseok and Kim, Kee-Eung},
  journal={Advances in Neural Information Processing Systems},
  volume={37},
  pages={109177--109201},
  year={2024}
}

@article{maoffline,
  title={Offline Imitation Learning upon Arbitrary Demonstrations by Pre-Training Dynamics Representations},
  author={Ma, Haitong and Dai, Bo and Ren, Zhaolin and Wang, Yebin and Li, Na},
  journal={arXiv preprint arXiv:2508.14383},
  year={2025}
}

@ARTICLE{IBCB,
author={Xu, Yi and Shen, Weiran and Xu, Jun and Zhang, Xiao and Wen, Ji-Rong},
journal={ IEEE Transactions on Pattern Analysis \& Machine Intelligence },
title={{ IBCB: Efficient Inverse Batched Contextual Bandit for Behavioral Evolution History }},
year={2026},
volume={},
number={01},
ISSN={1939-3539},
pages={1-17},
doi={10.1109/TPAMI.2026.3650796},
url = {https://doi.ieeecomputersociety.org/10.1109/TPAMI.2026.3650796},
publisher={IEEE Computer Society},
address={Los Alamitos, CA, USA},
month=jan}

@inproceedings{huyuk2022inverse,
  title={Inverse contextual bandits: Learning how behavior evolves over time},
  author={H{\"u}y{\"u}k, Alihan and Jarrett, Daniel and van der Schaar, Mihaela},
  booktitle={International Conference on Machine Learning},
  pages={9506--9524},
  year={2022},
  organization={PMLR}
}

@article{belkhale2023data,
  title={Data quality in imitation learning},
  author={Belkhale, Suneel and Cui, Yuchen and Sadigh, Dorsa},
  journal={Advances in neural information processing systems},
  volume={36},
  pages={80375--80395},
  year={2023}
}

@article{hoang2024sprinql,
  title={SPRINQL: Sub-optimal demonstrations driven offline imitation learning},
  author={Hoang, Huy and Mai, Tien and Varakantham, Pradeep},
  journal={Advances in Neural Information Processing Systems},
  volume={37},
  pages={136837--136872},
  year={2024}
}

@inproceedings{yue2024leverage,
  title={How to leverage diverse demonstrations in offline imitation learning},
  author={Yue, Sheng and Liu, Jiani and Hua, Xingyuan and Ren, Ju and Lin, Sen and Zhang, Junshan and Zhang, Yaoxue},
  booktitle={Proceedings of the 41st International Conference on Machine Learning},
  pages={58037--58067},
  year={2024}
}

@article{xu2024provably,
  title={Provably efficient offline reinforcement learning with trajectory-wise reward},
  author={Xu, Tengyu and Wang, Yue and Zou, Shaofeng and Liang, Yingbin},
  journal={IEEE Transactions on Information Theory},
  year={2024},
  publisher={IEEE}
}

@inproceedings{kang2023beyond,
  title={Beyond reward: offline preference-guided policy optimization},
  author={Kang, Yachen and Shi, Diyuan and Liu, Jinxin and He, Li and Wang, Donglin},
  booktitle={Proceedings of the 40th International Conference on Machine Learning},
  pages={15753--15768},
  year={2023}
}

@article{choi2024listwise,
  title={Listwise Reward Estimation for Offline Preference-based Reinforcement Learning},
  author={Choi, Heewoong and Jung, Sangwon and Ahn, Hongjoon and Moon, Taesup},
  journal={Proceedings of Machine Learning Research},
  volume={235},
  pages={8651--8671},
  year={2024},
  publisher={ML Research Press}
}

@article{massart2006risk,
  title={Risk bounds for statistical learning},
  author={Massart, Pascal and N{\'e}d{\'e}lec, {\'E}lodie},
  year={2006}
}

@article{natarajan1989learning,
  title={On learning sets and functions},
  author={Natarajan, Balas K},
  journal={Machine Learning},
  volume={4},
  number={1},
  pages={67--97},
  year={1989},
  publisher={Springer}
}

@article{zhang2004statistical,
  title={Statistical behavior and consistency of classification methods based on convex risk minimization},
  author={Zhang, Tong},
  journal={The Annals of Statistics},
  volume={32},
  number={1},
  pages={56--85},
  year={2004},
  publisher={Institute of Mathematical Statistics}
}

@article{tewari2007consistency,
  title={On the Consistency of Multiclass Classification Methods.},
  author={Tewari, Ambuj and Bartlett, Peter L},
  journal={Journal of Machine Learning Research},
  volume={8},
  number={5},
  year={2007}
}

\appendix
\section{Verification of Assumption~\ref{ass:massartnoise}}
\begin{figure*}[ht] 
  \centering

  \begin{subfigure}[b]{0.48\textwidth}
    \centering
    \includegraphics[width=\linewidth]{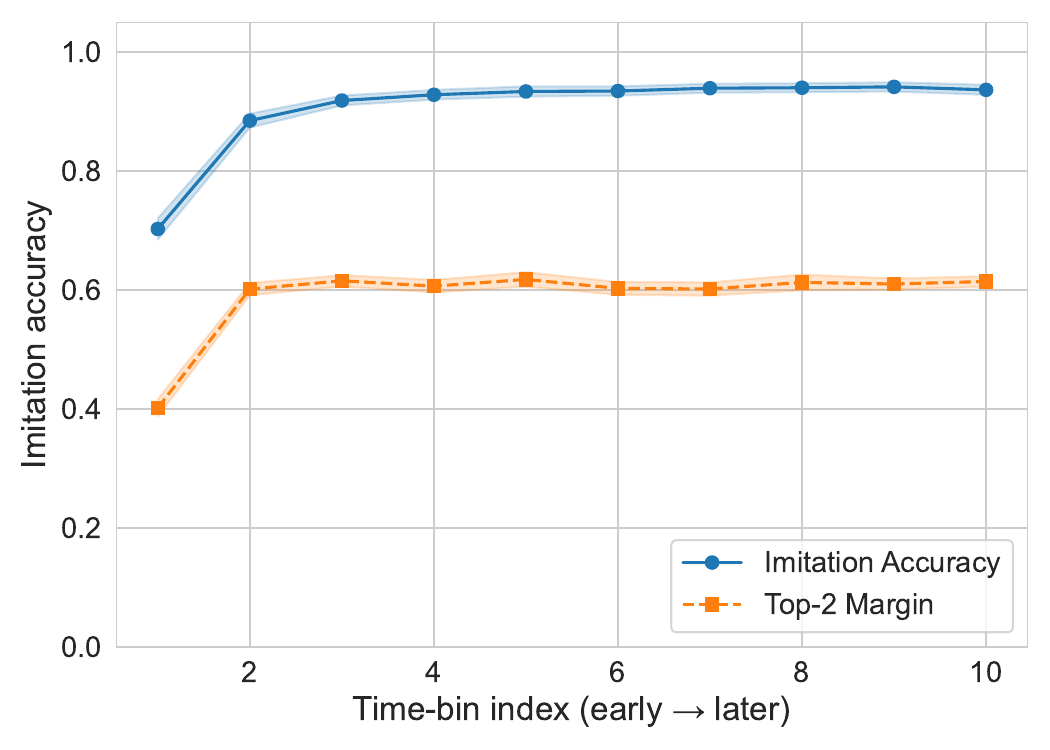}
    \caption{Predictability}
    \label{fig:predicbility}
  \end{subfigure}
  \hfill 
    \begin{subfigure}[b]{0.48\textwidth} 
    \centering
    \includegraphics[width=\linewidth]{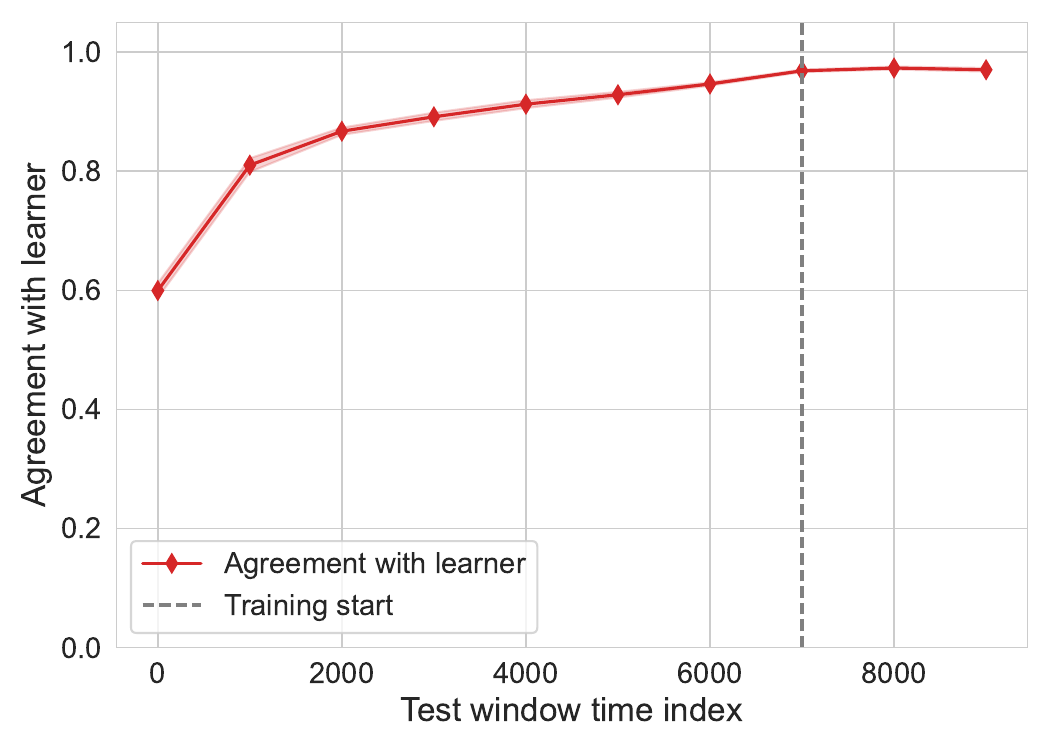}
    \caption{Late-Policy Generalizetion}
    \label{fig:latepolicy}
  \end{subfigure}
  \caption{Diagnostic verification of Assumption~\ref{ass:massartnoise}: (a) learner actions become increasingly predictable over time; (b) a late-trained observer agrees more with the learner on late test windows than on early ones (95\% CIs over 20 seeds).}. 
\end{figure*}

Assumption ~\ref{ass:massartnoise} concerns a post–burn-in regime in which the learner’s behavior becomes increasingly reliable, i.e., the effective decision noise decreases with the burn-in length $T(N)$. In real-world reward-free logs the optimal action $a_t^*$ is not observable, so the assumption cannot be verified in its definition. We therefore report a diagnostic verification that targets an observable implication: after discarding an initial prefix, the learner’s actions should become more predictable from ($\mathcal A_t, X_t$) and more stable across nearby late-stage time windows.

We partition the trajectory into $J$ equal time bins (earlier $\rightarrow$ later). Within each bin, we fit a fixed-capacity logit observer on 80\% of the samples and evaluate imitation accuracy on the remaining 20\%. Over 20 independent seeds, imitation accuracy exhibits a strong monotone increase over time (Figure \ref{fig:predicbility}): the Spearman correlation between time-bin index and mean accuracy is $r=0.9273$ with $p=1.12\times 10^{-4}$. Moreover, comparing the first $30\%$ bins (early) to the last $30\%$ bins (late), late-stage accuracy improves by $0.1041$ on average, with a $95\%$ confidence interval $[0.0939, 0.1142]$.

To assess temporal stability, we train a late-phase observer policy on the last $30\%$ of the trajectory and evaluate its action agreement with the learner on rolling test windows spanning the entire horizon. The resulting agreement is substantially higher on late windows than on early windows (Figure~\ref{fig:latepolicy}): the mean gap (late $\rightarrow$late $-$ late $\rightarrow$early) is $0.2119$, with $95\%$ confidence interval $[0.2055, 0.2182]$.

These diagnostics are consistent with the intended content of Assumption~\ref{ass:massartnoise}: after sufficient burn-in, the learner’s behavior becomes more predictable from ($\mathcal A_t, X_t$) and more stable across late-stage time windows, suggesting a reduction in effective behavioral noise relative to early stages.

\section{Theoretical Details}
\subsection{Proof of Lemma \ref{lem:uniform_convergence}}\label{appdix:proof5.5}
\begin{proof}
Denote the samples of phase \Rmnum{2} as $D_\mathrm{obs}^{T(N)+1:N}=\{(\mathcal A_t, X_t, \hat a_t)\}_{t=T(N)+1}^N$. The horizon is $L(N)=N-T(N)$. For any $\pi\in\Pi$, the 0-1 limitation loss is:
\[l_t(\pi)=\mathbf{1}[\pi(\mathcal A_t, X_t)\neq \hat a_t].\]

Recall the population and empirical noisy imitation risks:
\[
\mathcal R_{\mathrm{noisy}}(\pi):=\mathbb{E}\big[\ell_t(\pi)\big],\]
\[
\tilde{\mathcal R}_{\mathrm{noisy}}(\pi):=\frac{1}{L}\sum_{t=T+1}^{N}\ell_t(\pi),
\]
where the expectation is with respect to the data-generating process of the
suffix phase.\footnote{As discussed in the main text, contexts may be i.i.d.\ while labels $Y_t=\hat a_t$
can depend on history; the argument below uses a martingale concentration bound for bounded losses.}

Let $\{\mathcal{F}_t\}$ be the natural filtration generated by the history up to time $t$.
Define the martingale difference sequence
\[
X_t(\pi):=\ell_t(\pi)-\mathbb{E}\!\left[\ell_t(\pi)\mid \mathcal{F}_{t-1}\right],\qquad t=T+1,\dots,N.
\]
Then $\mathbb{E}[X_t(\pi)\mid \mathcal{F}_{t-1}]=0$ and $X_t(\pi)\in[-1,1]$ almost surely.
Hence, by the Azuma--Hoeffding inequality \cite{azuma1967weighted},
for any $\varepsilon>0$,
\[
\Pr\!\left(
\left|
\frac{1}{L}\sum_{t=T+1}^{N}X_t(\pi)
\right|>\varepsilon
\right)
\le
2\exp(-2L\varepsilon^2).
\]
Equivalently,
\[
\Pr\!\left(
\left|
\tilde {\mathcal R}_{\mathrm{noisy}}(\pi)
-
\frac{1}{L}\sum_{t=T+1}^{N}\mathbb{E}\!\left[\ell_t(\pi)\mid \mathcal{F}_{t-1}\right]
\right|>\varepsilon
\right)
\le
2\exp(-2L\varepsilon^2).
\]
Under the effective-distribution viewpoint used in our analysis, the time-average conditional expectation
matches $R_{\mathrm{noisy}}(\pi)$, so the same bound applies to
$\left|\tilde{\mathcal R}_{\mathrm{noisy}}(\pi)-\mathcal R_{\mathrm{noisy}}(\pi)\right|$.

Fix the instance sequence $((\mathcal A_{T+1},X_{T+1}),\dots,(\mathcal A_{N},X_{N}))$.
The values of $\{\ell_t(\pi)\}_{t=T+1}^{N}$ depend on $\pi$ only through the prediction vector
$(\pi(\mathcal A_{T+1}, X_t),\dots,\pi(\mathcal A_{N} X_{N}))$.
Let $\tau_\Pi(L)$ denote the multiclass growth function of $\Pi$ on $L$ points.
For classes with finite Natarajan dimension, there exists an absolute constant $c>0$ such that
\[
\tau_\Pi(L)\le (cL)^{d_{\mathrm{Nat}}(\Pi)}.
\]
This is a standard consequence of finite Natarajan dimension
\cite{natarajan1989learning, daniely2015multiclass}.

Applying a union bound over all distinct prediction vectors realized by $\Pi$ on the suffix instances,
and using the fixed-$\pi$ concentration from Step 1, we obtain
\[
\Pr\!\left(
\sup_{\pi\in\Pi}\left|\tilde{\mathcal R}_{\mathrm{noisy}}(\pi)-\mathcal R_{\mathrm{noisy}}(\pi)\right|>\varepsilon
\right)
\le
2\,\tau_\Pi(L)\exp(-2L\varepsilon^2)
\le
2(cL)^{d_{\mathrm{Nat}}(\Pi)}\exp(-2L\varepsilon^2).
\]
Setting the right-hand side to be at most $\delta$ and solving for $\varepsilon$ yields
\[
\sup_{\pi\in\Pi}\left|\tilde{\mathcal R}_{\mathrm{noisy}}(\pi)-\mathcal R_{\mathrm{noisy}}(\pi)\right|
\le
\sqrt{\frac{d_{\mathrm{Nat}}(\Pi)\log(cL)+\log(2/\delta)}{2L}}
\;\;\le\;\;
C\sqrt{\frac{d_{\mathrm{Nat}}(\Pi)\log L+\log(1/\delta)}{L}},
\]
where $C>0$ absorbs absolute constants. This proves the lemma.
\end{proof}
\section{Experimental Details}
\label{app:experimental_details}

In this section, we provide a comprehensive description of the experimental environment, the hyperparameter configurations for both the learner and the observer, and the optimization procedures used for training the observer's policy.

\subsection{Environmental Setup}
\label{app:env_setup}

We evaluate our framework on a synthetic linear contextual bandit environment. The environment is defined by the following specifications:

\begin{itemize}
    \item Context Generation: At each round $t$, the context vector $x_{t,a} \in \mathbb{R}^d$ for each arm $a \in [K]$ is drawn independently from a standard multivariate normal distribution, i.e., $x_{t,a} \sim \mathcal{N}(0, I_d)$. We then normalize each context vector such that $\|x_{t,a}\|_2 \le 1$.
    \item Reward Generation: A true parameter vector $\theta^* \in \mathbb{R}^d$ is sampled uniformly from the unit sphere $\mathcal{S}^{d-1}$. The reward for arm $a$ at round $t$ is generated as: $r_{t,a} = \langle x_{t,a}, \theta^* \rangle + \eta_t$,
    where the noise term $\eta_t$ is sampled from a Gaussian distribution $\mathcal{N}(0, \sigma^2)$ with $\sigma=0.1$.
    \item Dimensions: Unless otherwise stated, the default problem dimension is set to $d=50$, and the number of arms is set to $K=200$. The total horizon for the interaction log is $N=10,000$.
    \item Software \& Hardware: All experiments were implemented in Python 3.9 using NumPy and Scikit-learn.
\end{itemize}

\subsection{Algorithm Hyperparameters}
\label{app:hyperparameters}

Here we detail the specific configurations for the active learners generating the data and the observer model.

\paragraph{Learner Configurations.} 
We employ two standard no-regret algorithms as the data-generating agents:
\begin{itemize}
    \item LinUCB:
    The exploration parameter $\alpha_{\text{UCB}}$ controls the width of the confidence ellipsoid. We set $\alpha_{\text{UCB}} = 0.1$, which we found to provide a balanced exploration-exploitation trade-off in our environment. The confidence bound is calculated as $\text{UCB}_{t,a} = \langle x_{t,a}, \hat{\theta}_t \rangle + \alpha_{\text{UCB}} \sqrt{x_{t,a}^\top V_t^{-1} x_{t,a}}$.
    
    \item LinTS:
    The variance parameter $\nu$ controls the inflation of the posterior covariance. We set $\nu = \sigma^2 \cdot d$, following standard heuristics to ensure sufficient exploration. At each step, the agent samples $\tilde{\theta}_t \sim \mathcal{N}(\hat{\theta}_t, \nu V_t^{-1})$ and selects the arm maximizing $\langle x_{t,a}, \tilde{\theta}_t \rangle$.
\end{itemize}

\paragraph{Observer Configurations.}
The observer employs a Two-Phase Suffix Imitation strategy:
\begin{itemize}
    \item Burn-in Phase ($T$): As discussed in the main text, we evaluate three strategies:
    \begin{enumerate}
        \item Naive ($T=0$): Uses the full trajectory.
        \item Rule-based ($T \approx N^{0.9}$): Specifically, for $N=10,000$, we set $T = \lfloor 10000^{0.9} \rfloor = 3981$.
        \item Oracle (Best): We sweep the burn-in exponent $\alpha \in [0, 1]$ and report the minimum error achieved.
    \end{enumerate}
    \item Regularization ($\lambda$): For the observer's regularized maximum likelihood estimation (MLE), we use an $L_2$ regularization strength of $\lambda = 10^{-4}$. This value was selected via a coarse grid search $\{10^{-3}, 10^{-4}, 10^{-5}\}$ on a held-out validation set.
\end{itemize}

\subsection{Optimization Details}
\label{app:optimization}

The observer's core task is to estimate $\hat{\theta}$ by maximizing the likelihood of the learner's chosen actions on the suffix dataset $\mathcal{D}_{\text{suffix}} = \{(X_t, a_t)\}_{t=T+1}^N$.

Since the observer does not see rewards, we treat the problem as a multi-class classification task where the "label" is the arm chosen by the learner. We use multinomial logistic regression (Softmax regression) model, which acts as a convex surrogate for the 0-1 imitation loss.

\section{Additional Experimental Results}\label{appdix:additionalresults}
\begin{figure*}[ht] 
    \centering
    
    \begin{minipage}{0.49\textwidth}
        \centering
        \begin{subfigure}[b]{0.48\linewidth}
            \includegraphics[width=\linewidth]{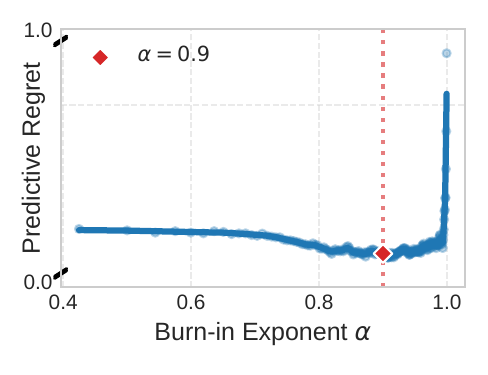}
            \caption{LinTS-Regret}
            \label{fig:regret_ts_row}
        \end{subfigure}
        \hfill
        \begin{subfigure}[b]{0.48\linewidth}
            \includegraphics[width=\linewidth]{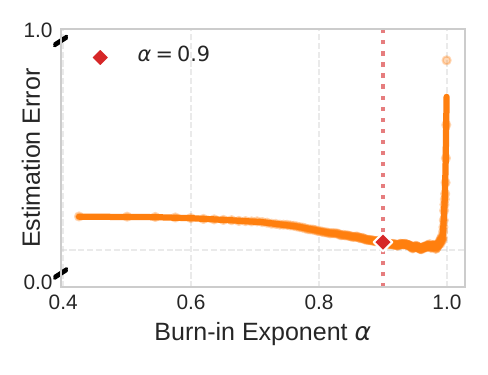}
            \caption{LinTS-Error}
            \label{fig:error_ts_row}
        \end{subfigure}
    \end{minipage}
    \hfill 
    \begin{minipage}{0.49\textwidth}
        \centering
        
        \begin{subfigure}[b]{0.48\linewidth}
            \includegraphics[width=\linewidth]{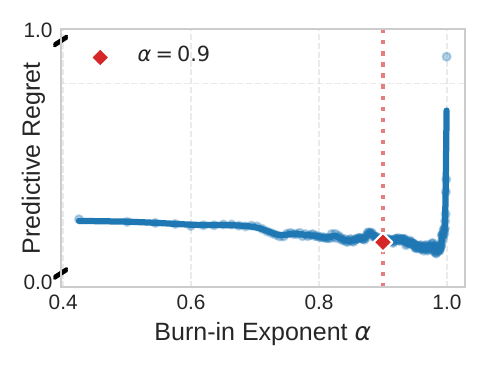}
            \caption{LinUCB-Regret}
            \label{fig:regret_ucb_row}
        \end{subfigure}
        \hfill
        \begin{subfigure}[b]{0.48\linewidth}
            \includegraphics[width=\linewidth]{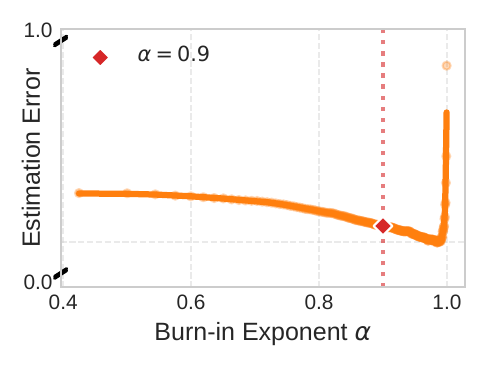}
            \caption{LinUCB-Error}
            \label{fig:error_ucb_row}
        \end{subfigure}
    \end{minipage}

    \caption{Performance comparison with full horizon layout ($d=20$, $K=50$)}
\end{figure*}

\begin{figure*}[ht] 
    \centering
    
    \begin{minipage}{0.49\textwidth}
        \centering
        \begin{subfigure}[b]{0.48\linewidth}
            \includegraphics[width=\linewidth]{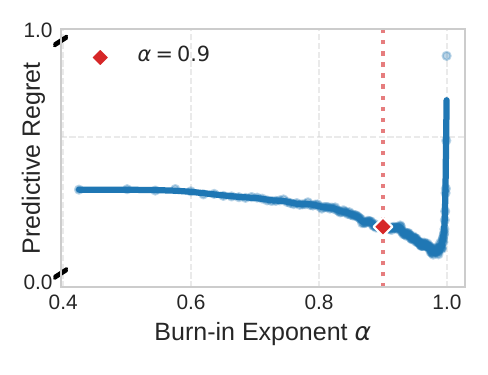}
            \caption{LinTS-Regret}
            \label{fig:regret_ts_row}
        \end{subfigure}
        \hfill
        \begin{subfigure}[b]{0.48\linewidth}
            \includegraphics[width=\linewidth]{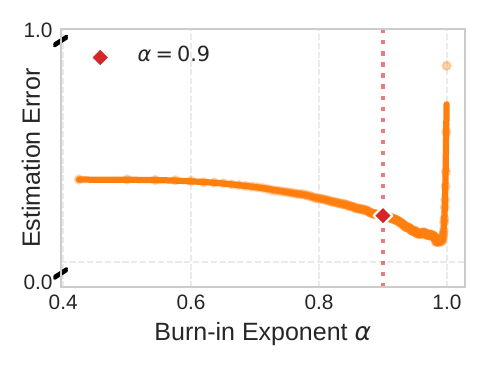}
            \caption{LinTS-Error}
            \label{fig:error_ts_row}
        \end{subfigure}
    \end{minipage}
    \hfill 
    \begin{minipage}{0.49\textwidth}
        \centering
        
        \begin{subfigure}[b]{0.48\linewidth}
            \includegraphics[width=\linewidth]{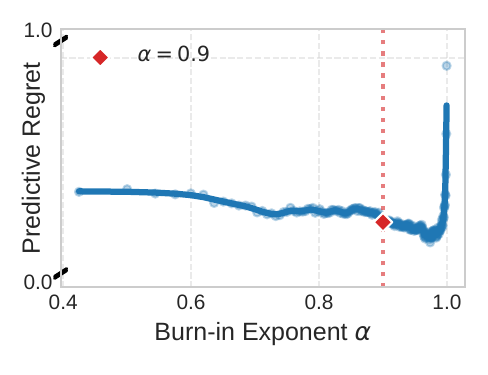}
            \caption{LinUCB-Regret}
            \label{fig:regret_ucb_row}
        \end{subfigure}
        \hfill
        \begin{subfigure}[b]{0.48\linewidth}
            \includegraphics[width=\linewidth]{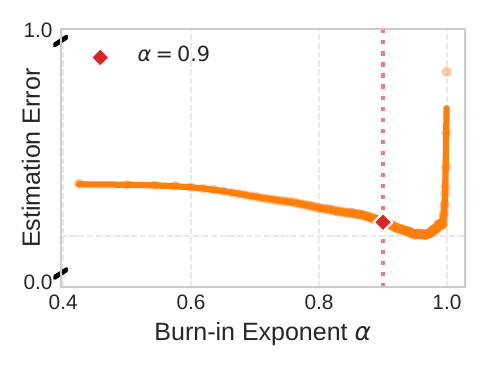}
            \caption{LinUCB-Error}
            \label{fig:error_ucb_row}
        \end{subfigure}
    \end{minipage}

    \caption{Performance comparison with full horizon layout ($d=20$, $K=100$)}
\end{figure*}

\begin{figure*}[ht] 
    \centering
    
    \begin{minipage}{0.49\textwidth}
        \centering
        \begin{subfigure}[b]{0.48\linewidth}
            \includegraphics[width=\linewidth]{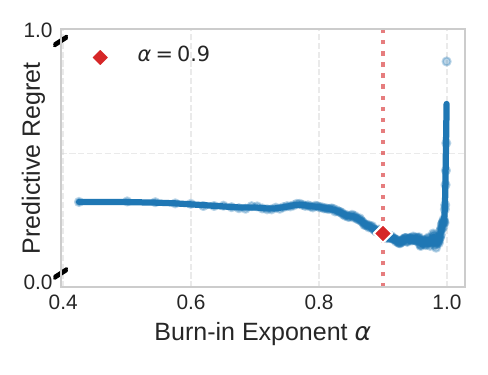}
            \caption{LinTS-Regret}
            \label{fig:regret_ts_row}
        \end{subfigure}
        \hfill
        \begin{subfigure}[b]{0.48\linewidth}
            \includegraphics[width=\linewidth]{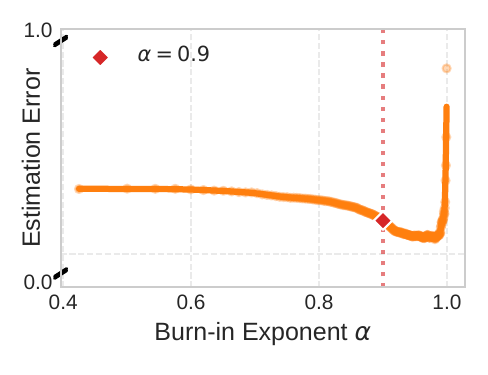}
            \caption{LinTS-Error}
            \label{fig:error_ts_row}
        \end{subfigure}
    \end{minipage}
    \hfill 
    \begin{minipage}{0.49\textwidth}
        \centering
        
        \begin{subfigure}[b]{0.48\linewidth}
            \includegraphics[width=\linewidth]{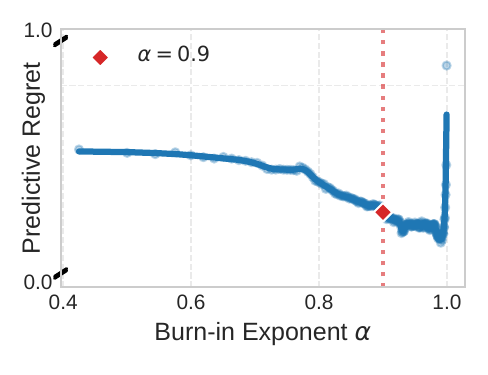}
            \caption{LinUCB-Regret}
            \label{fig:regret_ucb_row}
        \end{subfigure}
        \hfill
        \begin{subfigure}[b]{0.48\linewidth}
            \includegraphics[width=\linewidth]{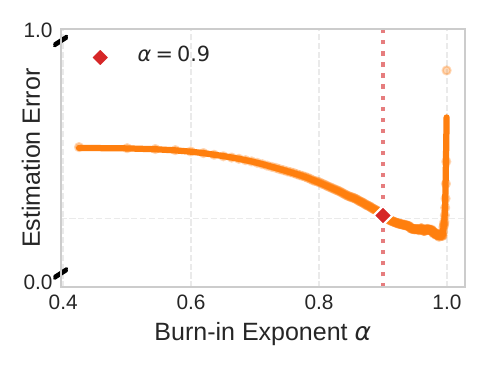}
            \caption{LinUCB-Error}
            \label{fig:error_ucb_row}
        \end{subfigure}
    \end{minipage}

    \caption{Performance comparison with full horizon layout ($d=20$, $K=200$)}
\end{figure*}

\begin{figure*}[ht] 
    \centering
    
    \begin{minipage}{0.49\textwidth}
        \centering
        \begin{subfigure}[b]{0.48\linewidth}
            \includegraphics[width=\linewidth]{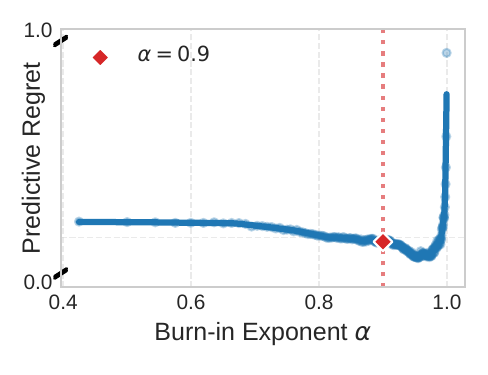}
            \caption{LinTS-Regret}
            \label{fig:regret_ts_row}
        \end{subfigure}
        \hfill
        \begin{subfigure}[b]{0.48\linewidth}
            \includegraphics[width=\linewidth]{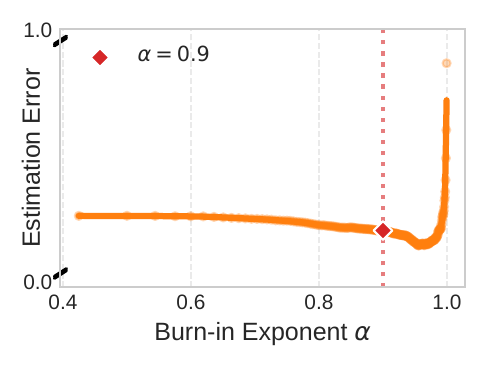}
            \caption{LinTS-Error}
            \label{fig:error_ts_row}
        \end{subfigure}
    \end{minipage}
    \hfill 
    \begin{minipage}{0.49\textwidth}
        \centering
        
        \begin{subfigure}[b]{0.48\linewidth}
            \includegraphics[width=\linewidth]{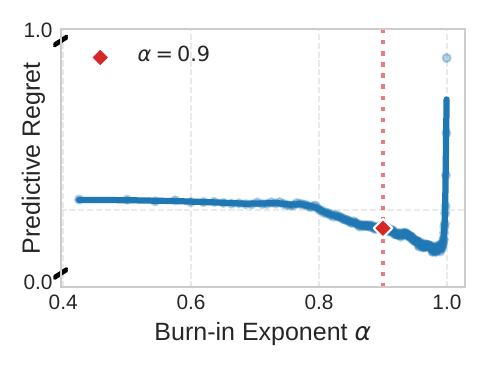}
            \caption{LinUCB-Regret}
            \label{fig:regret_ucb_row}
        \end{subfigure}
        \hfill
        \begin{subfigure}[b]{0.48\linewidth}
            \includegraphics[width=\linewidth]{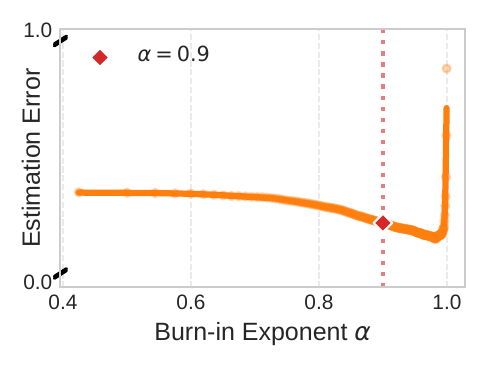}
            \caption{LinUCB-Error}
            \label{fig:error_ucb_row}
        \end{subfigure}
    \end{minipage}

    \caption{Performance comparison with full horizon layout ($d=50$, $K=50$)}
\end{figure*}

\begin{figure*}[ht] 
    \centering
    
    \begin{minipage}{0.49\textwidth}
        \centering
        \begin{subfigure}[b]{0.48\linewidth}
            \includegraphics[width=\linewidth]{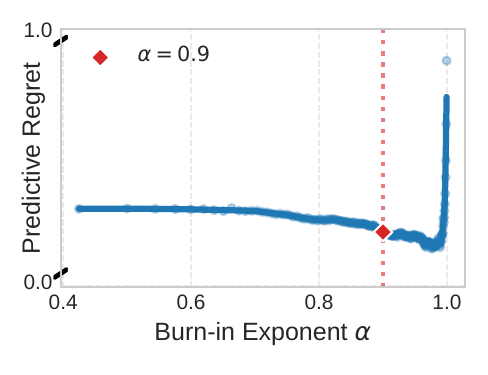}
            \caption{LinTS-Regret}
            \label{fig:regret_ts_row}
        \end{subfigure}
        \hfill
        \begin{subfigure}[b]{0.48\linewidth}
            \includegraphics[width=\linewidth]{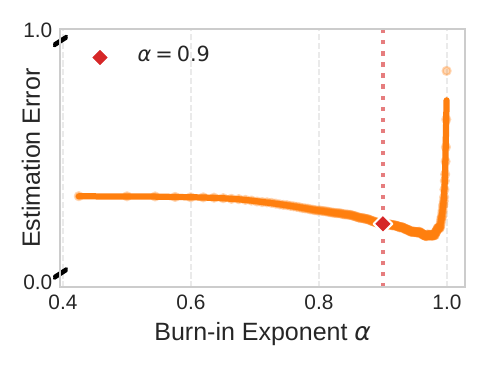}
            \caption{LinTS-Error}
            \label{fig:error_ts_row}
        \end{subfigure}
    \end{minipage}
    \hfill 
    \begin{minipage}{0.49\textwidth}
        \centering
        
        \begin{subfigure}[b]{0.48\linewidth}
            \includegraphics[width=\linewidth]{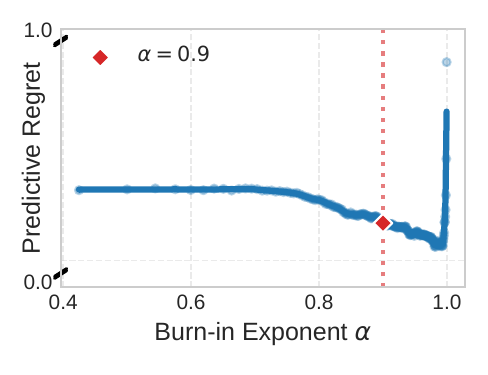}
            \caption{LinUCB-Regret}
            \label{fig:regret_ucb_row}
        \end{subfigure}
        \hfill
        \begin{subfigure}[b]{0.48\linewidth}
            \includegraphics[width=\linewidth]{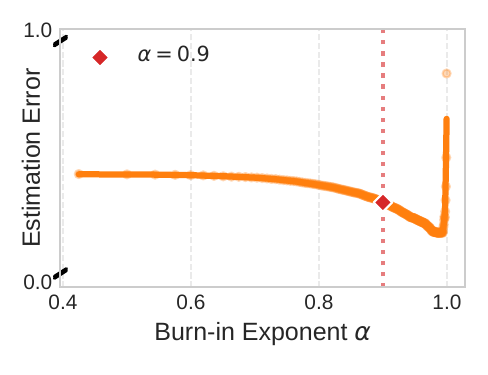}
            \caption{LinUCB-Error}
            \label{fig:error_ucb_row}
        \end{subfigure}
    \end{minipage}

    \caption{Performance comparison with full horizon layout ($d=50$, $K=100$)}
\end{figure*}


\begin{figure*}[t]
  \centering

  \begin{subfigure}[b]{0.48\textwidth}
    \centering
    \includegraphics[width=\linewidth]{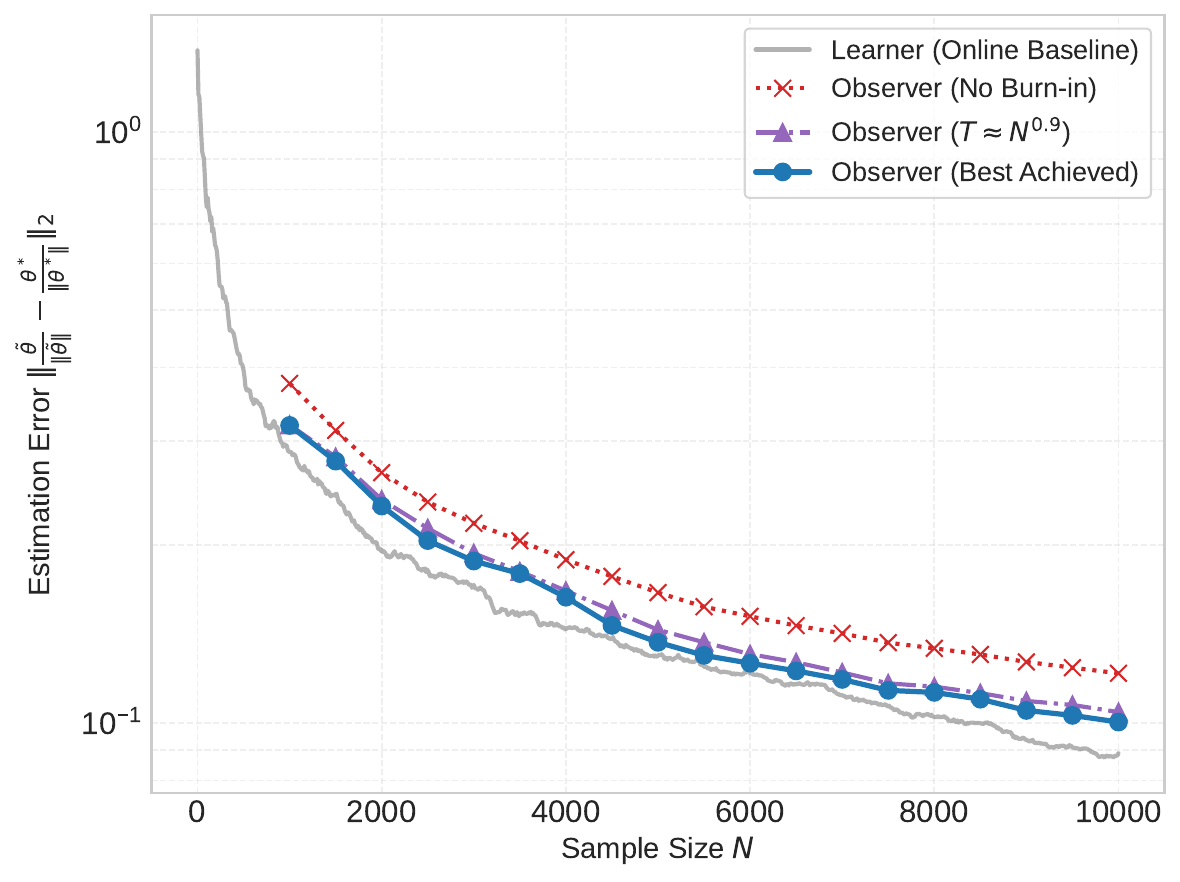}
    \caption{LinTS ($K=50$)}
    \label{fig:comp_ts_k50}
  \end{subfigure}
  \hfill
  \begin{subfigure}[b]{0.48\textwidth}
    \centering
    \includegraphics[width=\linewidth]{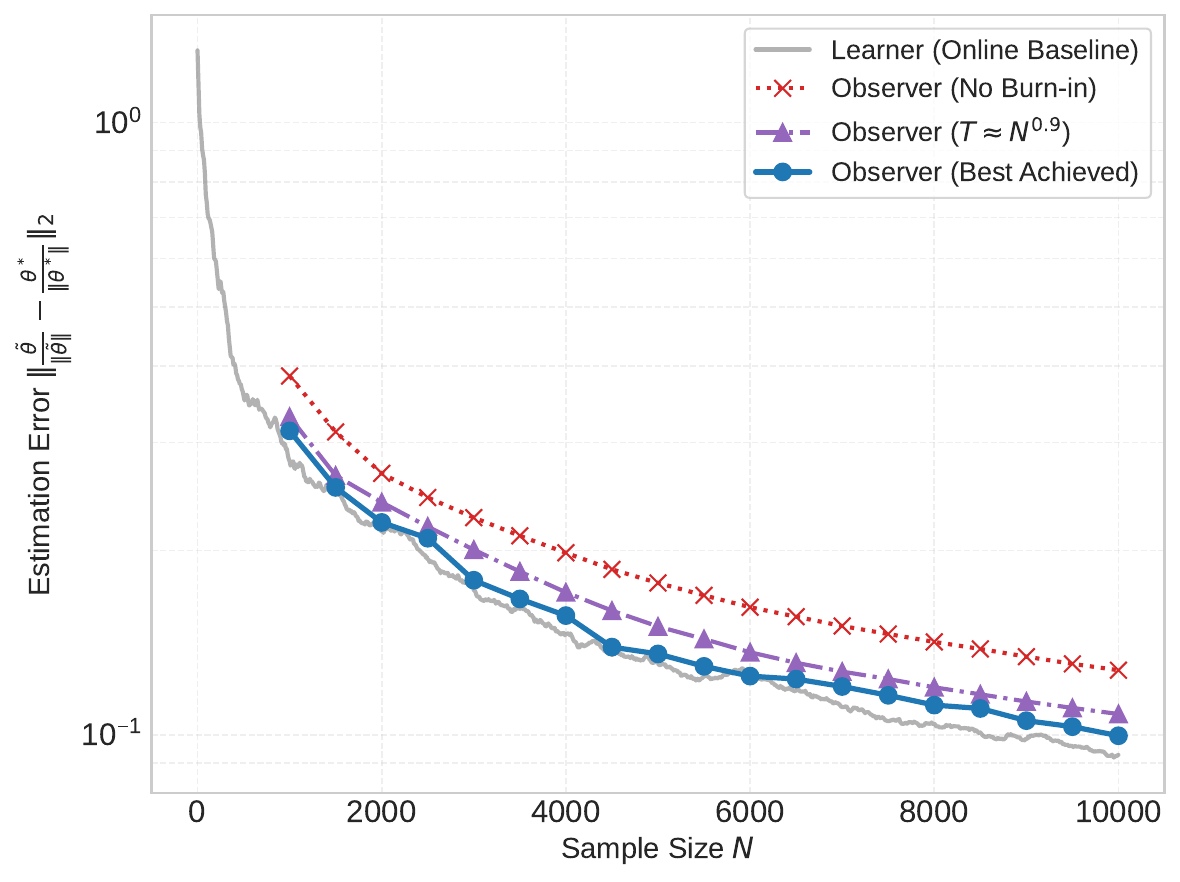}
    \caption{LinUCB ($K=50$)}
    \label{fig:comp_ucb_k50}
  \end{subfigure}

  \vspace{2mm} 

  \begin{subfigure}[b]{0.48\textwidth}
    \centering
    \includegraphics[width=\linewidth]{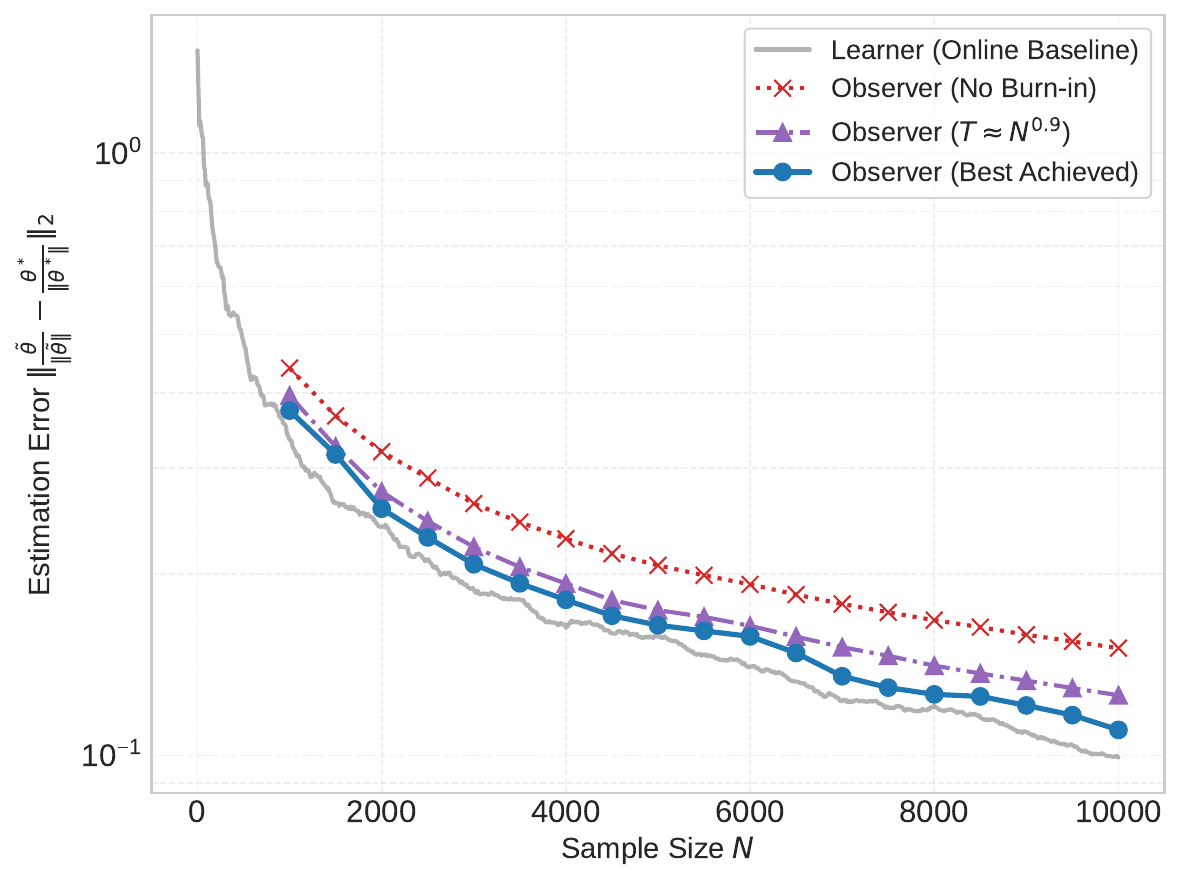}
    \caption{LinTS ($K=100$)}
    \label{fig:comp_ts_k100}
  \end{subfigure}
  \hfill
  \begin{subfigure}[b]{0.48\textwidth}
    \centering
    \includegraphics[width=\linewidth]{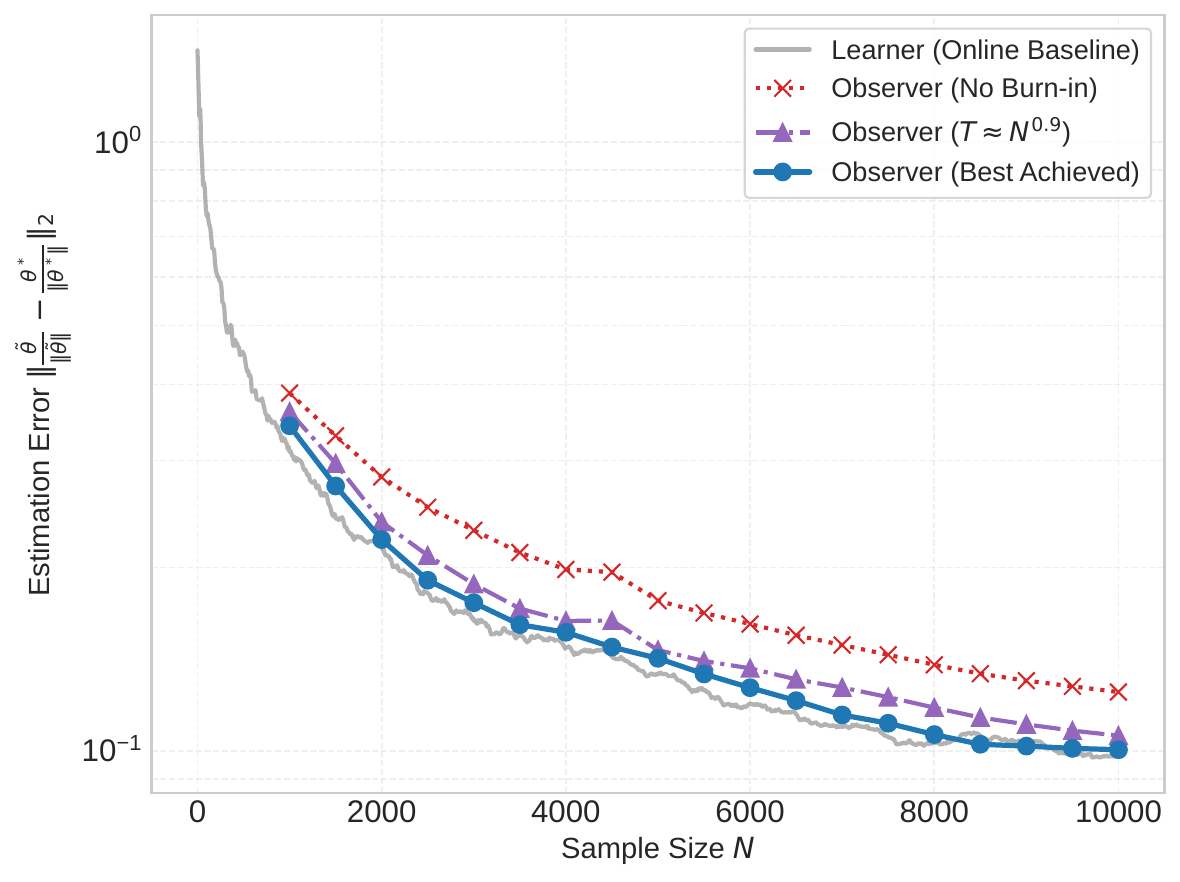}
    \caption{LinUCB ($K=100$)}
    \label{fig:comp_ucb_k100}
  \end{subfigure}

  \vspace{2mm}

  \begin{subfigure}[b]{0.48\textwidth}
    \centering
    \includegraphics[width=\linewidth]{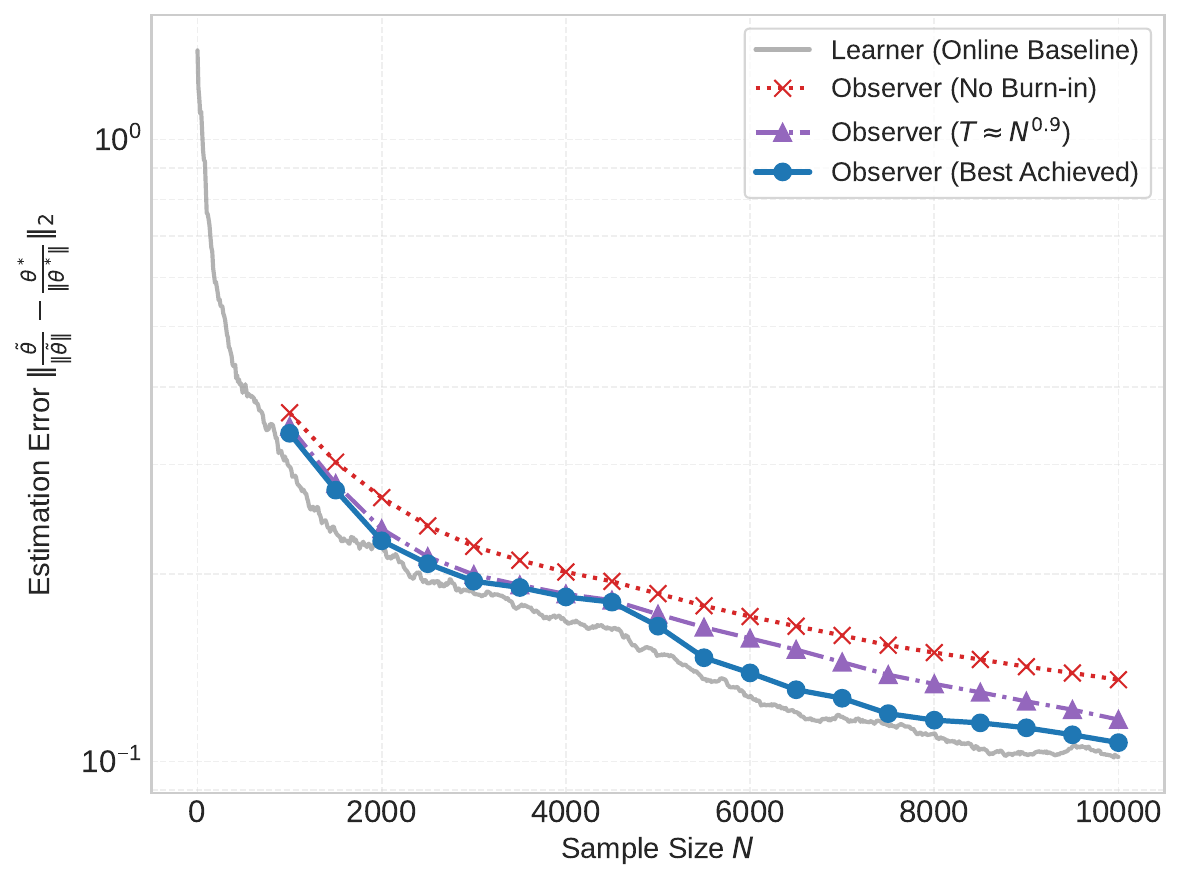}
    \caption{LinTS ($K=200$)}
    \label{fig:comp_ts_k200}
  \end{subfigure}
  \hfill
  \begin{subfigure}[b]{0.48\textwidth}
    \centering
    \includegraphics[width=\linewidth]{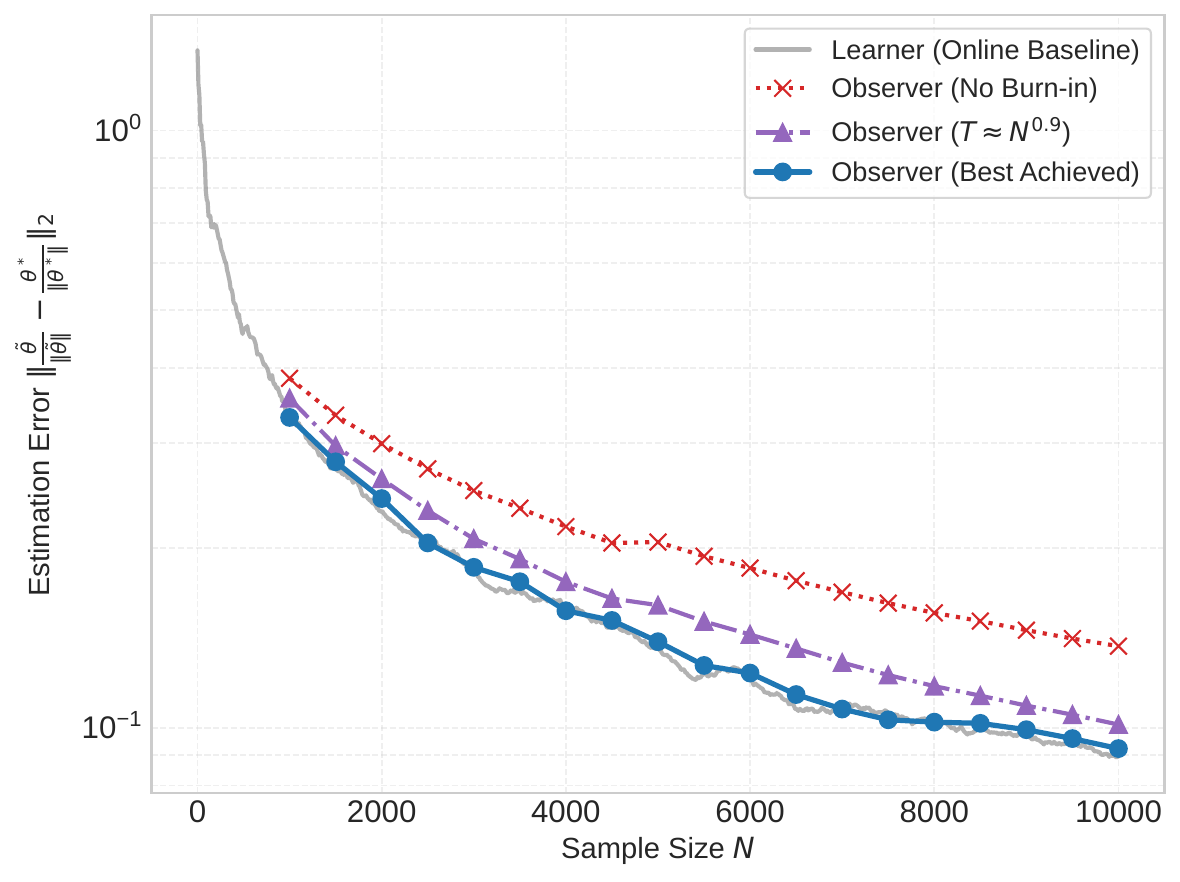}
    \caption{LinUCB ($K=200$)}
    \label{fig:comp_ucb_k200}
  \end{subfigure}

  \caption{Performance comparison of Observer strategies against the Learner baseline ($d=20$) under varying numbers of arms $K\in\{50,100,200\}$. Each row corresponds to a fixed $K$; left: LinTS, right: LinUCB.}
  \label{fig:observer_comparison_allK}
\end{figure*}


\begin{figure*}[t]
  \centering

  \begin{subfigure}[b]{0.48\textwidth}
    \centering
    \includegraphics[width=\linewidth]{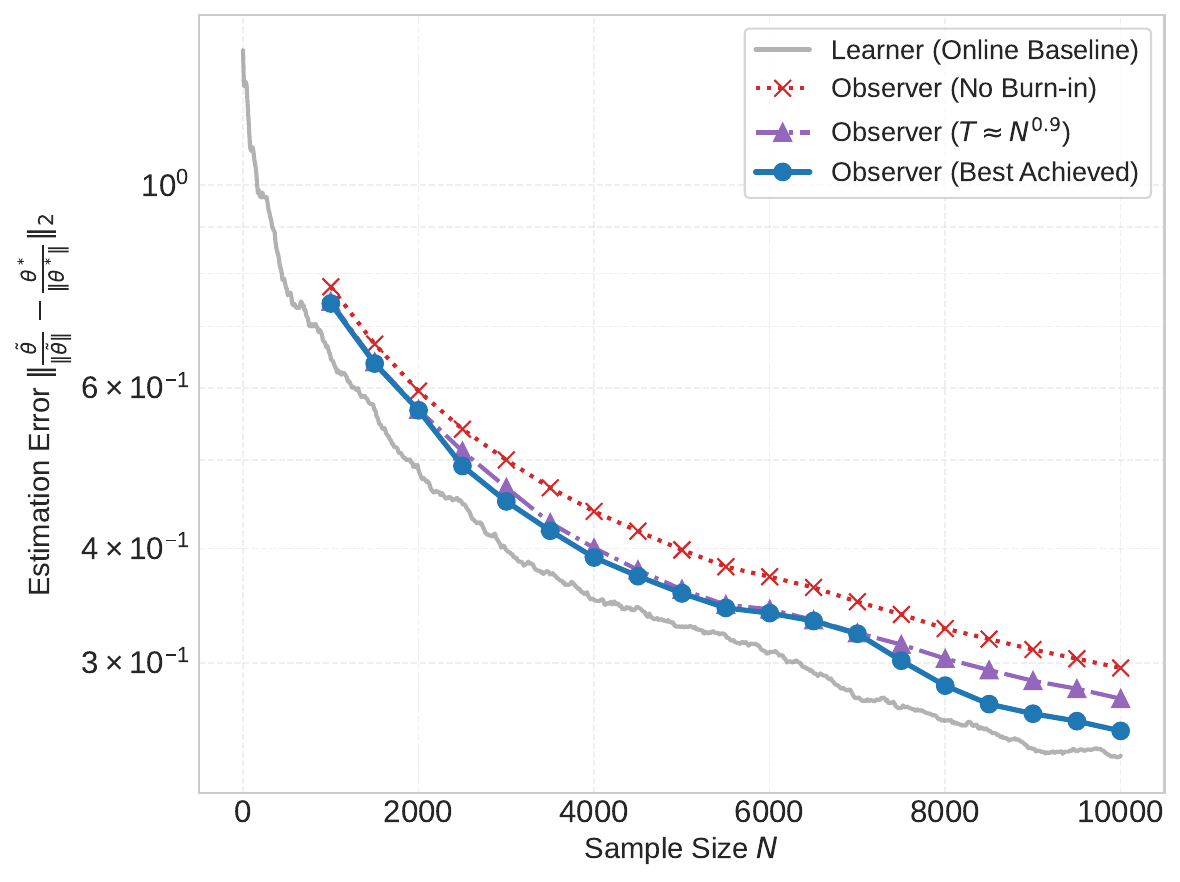}
    \caption{LinTS ($K=50$)}
    \label{fig:comp_ts_k50}
  \end{subfigure}
  \hfill
  \begin{subfigure}[b]{0.48\textwidth}
    \centering
    \includegraphics[width=\linewidth]{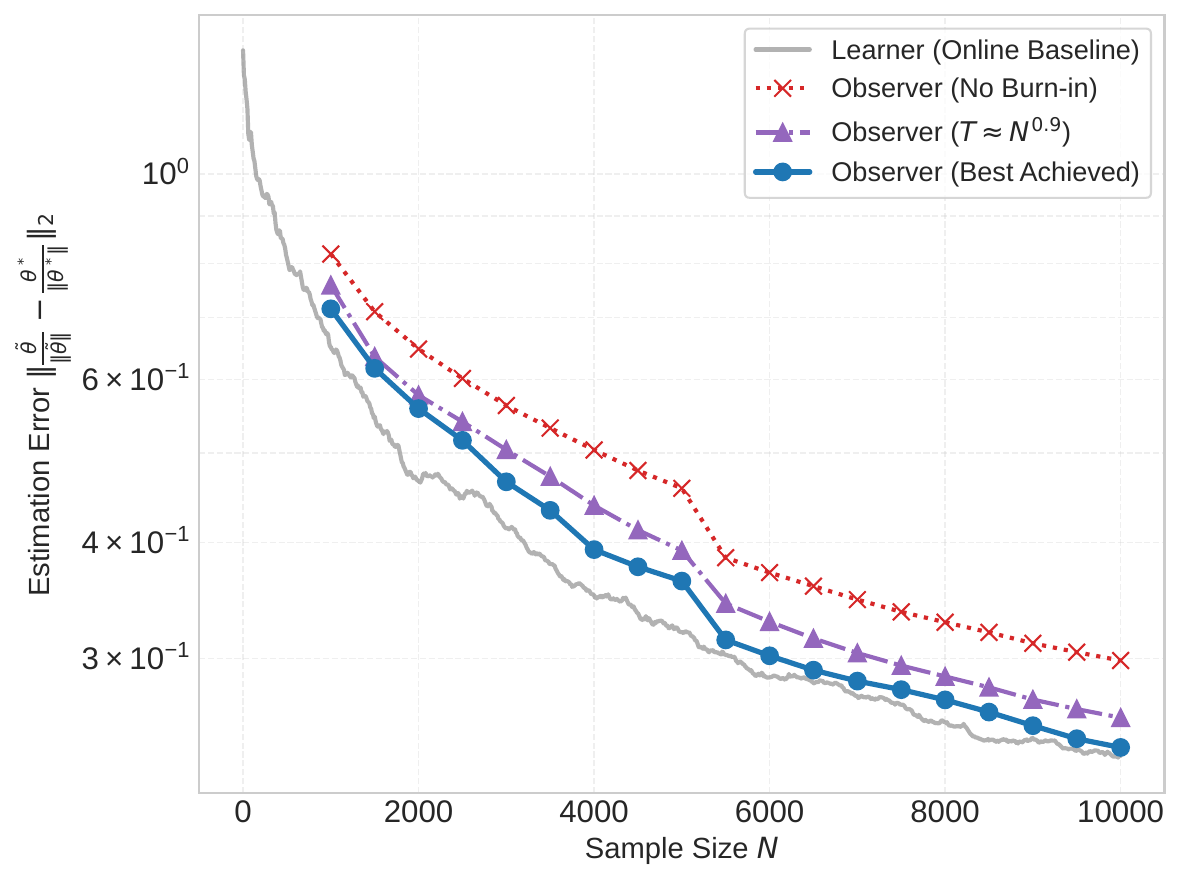}
    \caption{LinUCB ($K=50$)}
    \label{fig:comp_ucb_k50}
  \end{subfigure}

  \vspace{2mm} 

  \begin{subfigure}[b]{0.48\textwidth}
    \centering
    \includegraphics[width=\linewidth]{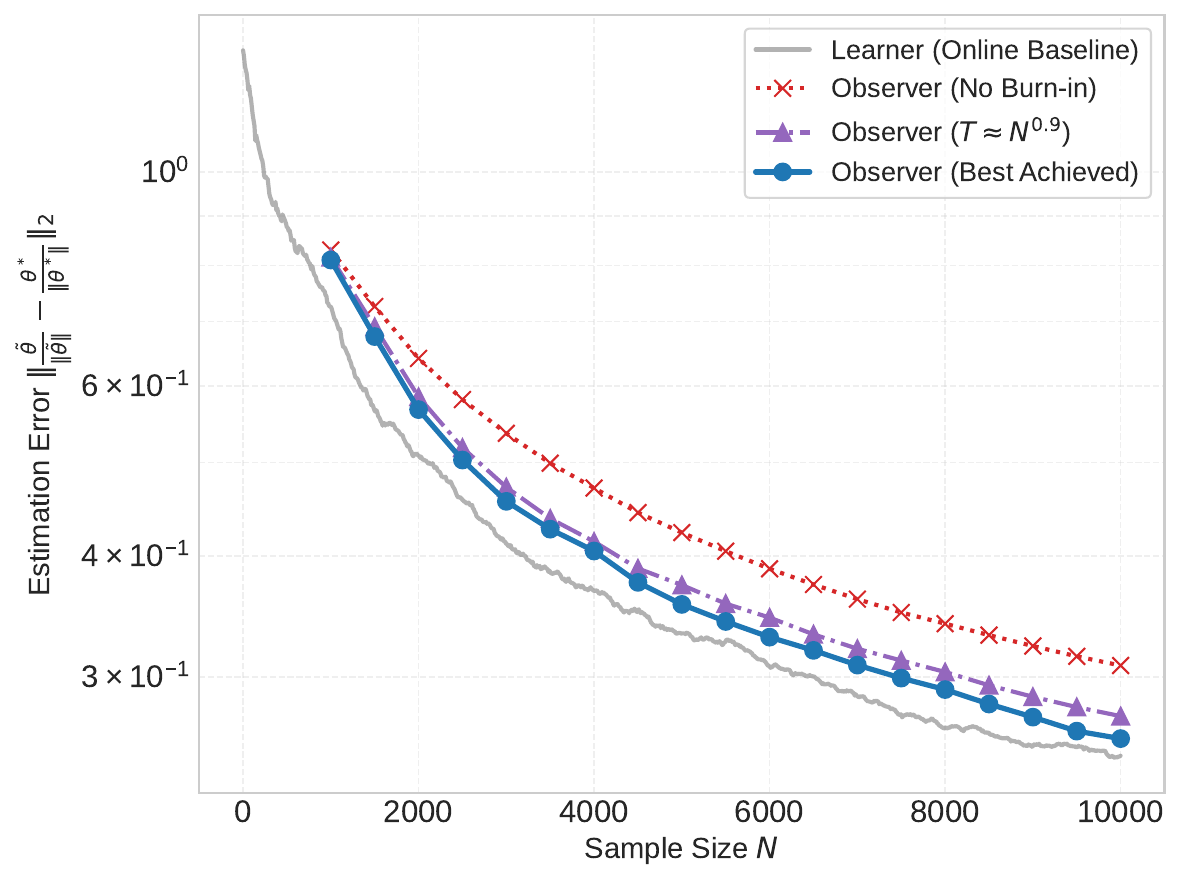}
    \caption{LinTS ($K=100$)}
    \label{fig:comp_ts_k100}
  \end{subfigure}
  \hfill
  \begin{subfigure}[b]{0.48\textwidth}
    \centering
    \includegraphics[width=\linewidth]{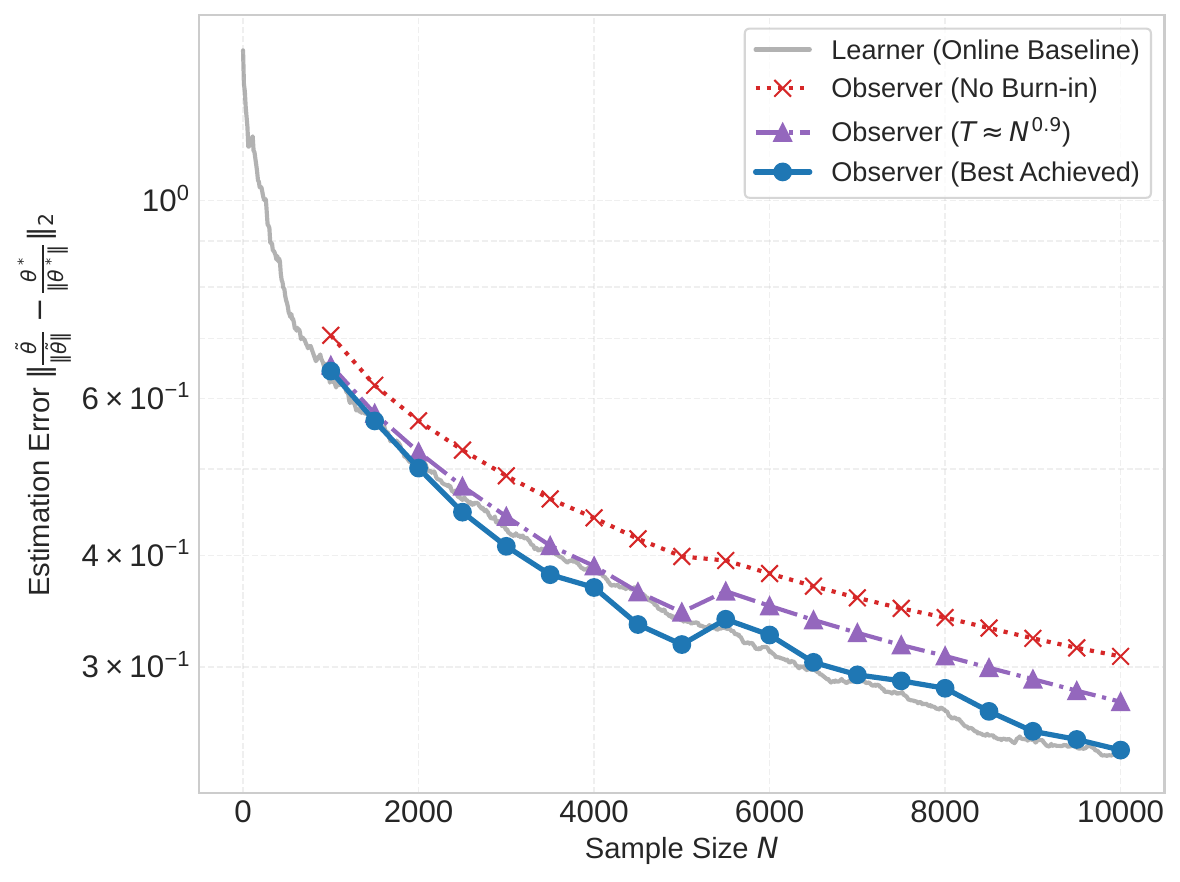}
    \caption{LinUCB ($K=100$)}
    \label{fig:comp_ucb_k100}
  \end{subfigure}

  \caption{Performance comparison of Observer strategies against the Learner baseline ($d=50$) under varying numbers of arms $K\in\{50,100\}$. Each row corresponds to a fixed $K$; left: LinTS, right: LinUCB.}
  \label{fig:observer_comparison_allK}
\end{figure*}

\end{document}